\documentclass{article} % For LaTeX2e
\usepackage{nips15submit_e}
\usepackage{xr}
\usepackage{epstopdf}
\usepackage{amsmath, amssymb, amsthm}
\usepackage{algorithm}
\usepackage{algorithmic}
\usepackage{times}
\title{Policy Gradient for Coherent Risk Measures}

\author{
Aviv Tamar \\
Electrical Engineering Department\\
The Technion - Israel Institute of Technology\\
\texttt{avivt@tx.technion.ac.il} \\
\And
Yinlam Chow \\
Institute for Computational \& \\Mathematical Engineering (ICME) \\
Stanford University\\
\texttt{ychow@stanford.edu} \\
\And
Mohammad Ghavamzadeh \\
Adobe Research \& INRIA\\
\texttt{ghavamza@adobe.com} \\
\And
Shie Mannor \\
Electrical Engineering Department\\
The Technion - Israel Institute of Technology\\
\texttt{shie@ee.technion.ac.il} \\
}

\newcommand{\citet}{\cite}
\newcommand{\citealt}{\cite}
\newcommand{\citep}{\cite}

\nipsfinalcopy % Uncomment for camera-ready version

\newcommand{\BEAS}{\begin{eqnarray*}}
\newcommand{\EEAS}{\end{eqnarray*}}
\newcommand{\BEQ}{\begin{equation}}
\newcommand{\EEQ}{\end{equation}}
\newcommand{\BIT}{\begin{itemize}}
\newcommand{\EIT}{\end{itemize}}

\usepackage{color}
\usepackage{amsmath}
\usepackage{amssymb}
\usepackage{graphicx}
\usepackage{comment,xspace}
\usepackage{fancybox}
\newcommand{\real}{{\mathbb{R}}}
\newcommand{\reals}{\real}

\newtheorem{theorem}{Theorem}[section]
\newtheorem{proposition}[theorem]{Proposition}
\newtheorem{lemma}[theorem]{Lemma}

\newtheorem{assumption}[theorem]{Assumption}

\newcommand{\pr}{P} % probability measure
\newcommand{\mdp}{\mathcal{M}} % policy parameter
\newcommand{\param}{\theta} % policy parameter
\newcommand{\dt}{\nabla_{\param}} % gradient
\newcommand{\dtN}{\nabla_{\param;N}} % sampled gradient
\newcommand{\pprob}{\pr_{\param}} % parameterized probability measure
\newcommand{\St}{{\mathcal{X}}} % state space
\newcommand{\Ac}{{\mathcal{A}}} % action space
\newcommand{\pol}{\mu_{\param}} % policy
\newcommand{\Exp}[1]{\mathbb{E}\left[ #1 \right]}
\newcommand{\ExpW}[2]{\mathbb{E}_{#2} \left[ {#1} \right]}
\newcommand{\U}{{\mathcal{U}}} % coherent risk set U
\newcommand{\cZ}{\mathcal Z}
\newcommand{\spset}{\mathcal S} % saddle point set
\newcommand{\quan}{q_\alpha} % quantile
\newcommand{\cvar}{\rho_{\text{CVaR}}} %cvar
\newcommand{\msd}{\rho_{\text{MSD}}} %mean semideviation
\newcommand{\SD}{\mathbb{SD}} %semideviation
\newcommand{\pemp}{\pr_{\param;N}} % empirical parameterized probability measure
\newcommand{\ind}[1]{\mathbb{I}\left\{ {#1} \right\}} % indicator function
\newcommand{\dotp}[2]{\langle {#1},{#2} \rangle} % dot product
\newcommand{\avgN}{\frac{1}{N}\sum_{i=1}^{N}} %empirical average
\newcommand{\saaspset}{\spset_{N}} % SAA saddle point set
\newcommand{\optset}{\mathcal P} % optimal solution set
\newcommand{\saaoptset}{\mathcal P_N} % SAA optimal solution set
\newcommand{\setdist}[2]{\mathbb D \left\{ {#1}, {#2}\right\}} % set distance
\newcommand{\saaLag}{L_{\theta;N}(\xi^*_{\theta;N},\lambda^{*,\mathcal P}_{\theta;N},\lambda^{*,\mathcal E}_{\theta;N},\lambda^{*,\mathcal I}_{\theta;N})} % SAA Lagrangian
\newcommand{\Lag}{L_{\theta}(\xi^*_{\theta},\lambda^{*,\mathcal P}_{\theta},\lambda^{*,\mathcal E}_{\theta},\lambda^{*,\mathcal I}_{\theta})} % Lagrangian
\newcommand{\alggiven}{\textbf{Given:}}
\newcommand{\algreturn}{\textbf{Return:}} 

\begin{document}

\maketitle

\begin{abstract}
Several authors have recently developed risk-sensitive policy gradient methods that augment the standard expected cost minimization problem with a measure of \emph{variability} in cost. These studies have focused on \emph{specific} risk-measures, such as the variance or conditional value at risk (CVaR). In this work, we extend the policy gradient method to \emph{the whole class} of coherent risk measures, which is widely accepted in finance and operations research, among other fields. We consider both static and time-consistent dynamic risk measures. For static risk measures, our approach is in the spirit of \emph{policy gradient} algorithms and combines a standard sampling approach with convex programming. For dynamic risk measures, our approach is \emph{actor-critic} style and involves explicit approximation of value function. Most importantly, our contribution presents a \emph{unified} approach to risk-sensitive reinforcement learning that generalizes and extends previous results.
\end{abstract}

%%%%%%%%%%%%%%%%%%%%%%%%%%%%%%%%%%%%%%%%%%%%%%%%%%%%%%%%%%%%%%%%%%%%%%%%%%%%%%%%
%%%%%%%%%%%%%%%%%%%%%%%%%%%%%%%%%%%%%%%%%%%%%%%%%%%%%%%%%%%%%%%%%%%%%%%%%%%%%%%%
%%%%%%%%%%%%%%%%%%%%%%%%%%%%%%%%%%%%%%%%%%%%%%%%%%%%%%%%%%%%%%%%%%%%%%%%%%%%%%%%
%%%%%%%%%%%%%%%%%%%%%%%%%%%%%%%%%%%%%%%%%%%%%%%%%%%%%%%%%%%%%%%%%%%%%%%%%%%%%%%%
%%%%%%%%%%%%%%%%%%%%%%%%%%%%%%%%%%%%%%%%%%%%%%%%%%%%%%%%%%%%%%%%%%%%%%%%%%%%%%%%

\section{Introduction}
% Risk is important, also for MDPs
Risk-sensitive optimization considers problems in which the objective involves a \emph{risk measure} of the random cost, in contrast to the typical \emph{expected} cost objective. Such problems are important when the decision-maker wishes to manage the \emph{variability} of the cost, in addition to its expected outcome, and are standard in various applications of finance and operations research. In reinforcement learning (RL) \cite{sutton_reinforcement_1998}, risk-sensitive objectives have gained popularity as a means to regularize the variability of the total (discounted) cost/reward in a Markov decision process (MDP).

% Many risk measures have been studied, Coherent risk is a popular *unified* approach
Many risk objectives have been investigated in the literature and applied to RL, such as the celebrated Markowitz mean-variance model~\cite{Markowitz59PS}, Value-at-Risk (VaR) and Conditional Value at Risk (CVaR)~\citep{moody2001learning,tamar2012policy,prashanth2013actor,delage_percentile_2010,chow2014cvar,tamar2015optimizing}. The view taken in this paper is that the preference of one risk measure over another is \emph{problem-dependent} and depends on factors such as the cost distribution, sensitivity to rare events, ease of estimation from data, and computational tractability of the optimization problem. However, the highly influential paper of Artzner et al.~\citet{artzner1999coherent} identified a set of natural properties that are desirable for a risk measure to satisfy. Risk measures that satisfy these properties are termed \emph{coherent} and have obtained widespread acceptance in financial applications, among others. We focus on such coherent measures of risk in this work.

% time consistency is also important
For sequential decision problems, such as MDPs, another desirable property of a risk measure is \emph{time consistency}. A time-consistent risk measure satisfies a ``dynamic programming" style property: if a strategy is risk-optimal for an $n$-stage problem, then the component of the policy from the $t$-th time until the end (where $t<n$) is also risk-optimal (see principle of optimality in~\citealt{Ber2012DynamicProgramming}). The recently proposed class of dynamic Markov coherent risk measures~\citep{ruszczynski2010risk} satisfies both the coherence and time consistency properties.

% We extend RL to coherent risk
In this work, we present policy gradient algorithms for RL with a coherent risk objective. Our approach applies to \emph{the whole class} of coherent risk measures, thereby generalizing and unifying previous approaches that have focused on individual risk measures.  We consider both \emph{static} coherent risk of the total discounted return from an MDP and  time-consistent {\em dynamic} Markov coherent risk.
%
%Our proposed algorithm for the static risk is in the spirit of \emph{policy gradient} algorithms~\citep{baxter2001infinite}, while the one for the dynamic risk is \emph{actor-critic} style~\citep{konda2000actor}.
%
Our main contribution is formulating the risk-sensitive policy-gradient under the coherent-risk framework. More specifically, we provide:
%\vspace{-10pt}
\begin{itemize}
\item A new formula for the gradient of static coherent risk that is convenient for approximation using sampling.
\item An algorithm for the gradient of general static coherent risk that involves sampling with convex programming and a corresponding consistency result.
\item A new policy gradient theorem for Markov coherent risk, relating the gradient to a suitable \emph{value function} and a corresponding actor-critic algorithm.
%\item A corresponding actor-critic algorithm for the gradient of dynamic Markov coherent risk, with function approximation in the value function. We prove consistency of the gradient, and analyze sensitivity to approximation errors in the value-function.
\end{itemize}
Several previous results are special cases of the results presented here; our approach allows to re-derive them in greater generality and simplicity.

\paragraph{Related Work}
Risk-sensitive optimization in RL for specific risk functions has been studied recently by several authors.~\citet{borkar2001sensitivity} studied exponential utility functions,~\citet{moody2001learning},~\citet{tamar2012policy},~\citet{prashanth2013actor} studied mean-variance models,~\citet{chow2014cvar},~\citet{tamar2015optimizing} studied CVaR in the static setting, and~\citet{petrik2012approximate},~\citet{chow_mpc_14} studied dynamic coherent risk for systems with linear dynamics. Our paper presents a general method \emph{for the whole class} of coherent risk measures (both static and dynamic) and is not limited to a specific choice within that class, nor to particular system dynamics.

Reference~\cite{osogami2012robustness} showed that an MDP with a dynamic coherent risk objective is essentially a robust MDP. The planning for large scale MDPs was considered in ~\citet{tamar2014robust}, using an approximation of the value function. For many problems, approximation in the policy space is more suitable (see, e.g.,~\citealt{MarTsi98}). Our sampling-based RL-style approach is suitable for approximations both in the policy and value function, and scales-up to large or continuous MDPs. We do, however, make use of a technique of~\citet{tamar2014robust} in a part of our method.

Optimization of coherent risk measures was thoroughly investigated by Ruszczynski and Shapiro~\cite{ruszczynski2006optimization} (see also~\citealt{Shapiro2009}) for the stochastic programming case in which the policy parameters do not affect the distribution of the stochastic system (i.e.,~the MDP trajectory), but only the reward function, and thus, this approach is not suitable for most RL problems. For the case of MDPs and dynamic risk,~\citet{ruszczynski2010risk} proposed a dynamic programming approach. This approach does not scale-up to large MDPs, due to the ``curse of dimensionality". For further motivation of risk-sensitive policy gradient methods, we refer the reader to~\citet{moody2001learning,tamar2012policy,prashanth2013actor,chow2014cvar,tamar2015optimizing}.

%In particular, for the special case of CVaR, we obtain similar results to~\citet{tamar2015optimizing}, but under weaker assumptions and simpler derivations.

%%%%%%%%%%%%%%%%%%%%%%%%%%%%%%%%%%%%%%%%%%%%%%%%%%%%%%%%%%%%%%%%%%%%%%%%%%%%%%%%
%%%%%%%%%%%%%%%%%%%%%%%%%%%%%%%%%%%%%%%%%%%%%%%%%%%%%%%%%%%%%%%%%%%%%%%%%%%%%%%%
%%%%%%%%%%%%%%%%%%%%%%%%%%%%%%%%%%%%%%%%%%%%%%%%%%%%%%%%%%%%%%%%%%%%%%%%%%%%%%%%
%%%%%%%%%%%%%%%%%%%%%%%%%%%%%%%%%%%%%%%%%%%%%%%%%%%%%%%%%%%%%%%%%%%%%%%%%%%%%%%%
%%%%%%%%%%%%%%%%%%%%%%%%%%%%%%%%%%%%%%%%%%%%%%%%%%%%%%%%%%%%%%%%%%%%%%%%%%%%%%%%

%\vspace{-0.1in}
\section{Preliminaries}
\label{sec:background}
%\vspace{-0.1in}

Consider a probability space $(\Omega, \mathcal F,\pprob)$, where $\Omega$ is the set of outcomes (sample space), $\mathcal F$ is a $\sigma$-algebra over $\Omega$ representing the set of events we are interested in, and $\pprob \in \mathcal B$, where $\mathcal B:=\left\{ \xi: \int_{\omega\in\Omega} \xi(\omega)=1, \xi\geq 0 \right\}$ is the set of probability distributions, is a probability measure over $\mathcal F$ parameterized by some tunable parameter $\theta \in \mathbb R^{K}$. In the following, we suppress the notation of $\param$ in $\param$-dependent quantities.

To ease the technical exposition, in this paper we restrict our attention to finite probability spaces, i.e.,~$\Omega$ has a finite number of elements. Our results can be extended to the $L_p$-normed spaces without loss of generality, but the details are omitted for brevity.

Denote by $\cZ$ the space of random variables $Z:\Omega\mapsto (-\infty,\infty)$ defined over the probability space $(\Omega, \mathcal F, \pprob)$. In this paper, a random variable $Z\in \cZ$ is interpreted as a cost, i.e.,~the smaller the realization of $Z$, the better. For $Z,W\in\mathcal Z$, we denote by $Z\leq W$ the point-wise partial order, i.e.,~$Z(\omega)\leq W(\omega)$ for all $\omega\in \Omega$. We denote by $\mathbb E_{\xi}[Z]\doteq \sum_{\omega\in\Omega}\pprob(\omega)\xi(\omega)Z(\omega)$ a $\xi$-weighted expectation of $Z$.%We now give a brief definition of MDPs.

An MDP is a tuple $\mdp=(\St,\Ac,C,P,\gamma,x_0)$, where $\St$ and $\Ac$ are the state and action spaces; $C(x)\in[-C_{\max},C_{\max}]$ is a bounded, deterministic, and state-dependent cost; $P(\cdot|x,a)$ is the transition probability distribution; $\gamma$ is a discount factor; and $x_0$ is the initial state.\footnote{Our results may easily be extended to random costs, state-action dependent costs, and random initial states.} Actions are chosen according to a $\param$-parameterized stationary Markov\footnote{For the dynamic Markov risk we study, an optimal policy is stationary Markov, while this is not necessarily the case for the static risk. Our results can be extended to history-dependent policies or stationary Markov policies on a state space augmented with the accumulated cost. The latter has shown to be sufficient for optimizing the CVaR risk~\cite{bauerle2011markov}.} policy $\pol(\cdot|x)$. We denote by $x_0,a_0,\dots,x_T,a_T$ a trajectory of length $T$ drawn by following the policy $\pol$ in the MDP.

%%%%%%%%%%%%%%%%%%%%%%%%%%%%%%%%%%%%%%%%%%%%%%%%%%%%%%%%%%%%%%%%%%%%%%%%%%%%%%%%
%%%%%%%%%%%%%%%%%%%%%%%%%%%%%%%%%%%%%%%%%%%%%%%%%%%%%%%%%%%%%%%%%%%%%%%%%%%%%%%%
%%%%%%%%%%%%%%%%%%%%%%%%%%%%%%%%%%%%%%%%%%%%%%%%%%%%%%%%%%%%%%%%%%%%%%%%%%%%%%%%

\vspace{-0.1in}
\subsection{Coherent Risk Measures}
\label{subsec:coherent}
\vspace{-0.1in}
% Static coherent riskier
A \emph{risk measure} is a function $\rho:\cZ \to \mathbb R$ that maps an uncertain outcome $Z$ to the extended real line $\reals \cup\{ +\infty,-\infty\}$, e.g.,~the expectation $\Exp{Z}$ or the conditional value-at-risk (CVaR) $\min_{\nu\in\mathbb R}\big\{\nu + \frac{1}{\alpha}\mathbb E\big[(Z-\nu)^+\big]\big\}$.
%   - Coherent risk axioms
A risk measure is called \emph{coherent}, if it satisfies the following conditions for all $Z,W\in\mathcal Z$~\cite{artzner1999coherent}:
\begin{description}
\item[A1] Convexity: $\forall\lambda\in[0,1],\;\rho\big(\lambda Z + (1-\lambda)W\big)\leq \lambda\rho(Z) + (1-\lambda)\rho(W)$;
\item[A2] Monotonicity:  if $Z\leq W$, then $\rho(Z)\leq\rho(W)$;
\item[A3] Translation invariance: $\forall a\! \in \!\mathbb R,\;\rho(Z+a)=\rho(Z) + a$;
\item[A4] Positive homogeneity: if $\lambda\geq0$, then $\rho(\lambda Z) = \lambda \rho(Z)$.
\end{description}
Intuitively, these condition ensure the ``rationality" of single-period risk assessments: A1 ensures that diversifying an investment will reduce its risk; A2 guarantees that an asset with a higher cost for every possible scenario is indeed riskier; A3, also known as `cash invariance', means that the deterministic part of an investment portfolio does not contribute to its risk; the intuition behind A4 is that doubling a position in an asset doubles its risk.
We further refer the reader to~\citet{artzner1999coherent} for a more detailed motivation of coherent risk.
%   - Coherent risk as min-max problem over the set U

The following representation theorem~\cite{Shapiro2009} shows an important property of coherent risk measures that is fundamental to our gradient-based approach.
\begin{theorem}
\label{thm:rep}
A risk measure $\rho:\mathcal Z \rightarrow \mathbb R$ is coherent if and only if there exists a convex bounded and closed set $\U \subset \mathcal B$ such that\footnote{When we study risk in MDPs, the risk envelop $\U(P_\theta)$ in Eq.~\ref{eq:coherent_as_optimization} also depends on the state $x$.
}%We drop this dependency for simplicity in the case of static risk analysis.}
\begin{equation}
\label{eq:coherent_as_optimization}
\rho(Z)=\max_{\xi\,:\,\xi P_\theta\in \U(P_\theta)} \mathbb E_{\xi}[Z].
\end{equation}
\end{theorem}
%
%TODO - do we need this weakly* closed comment?
\vspace{-5pt}
The result essentially states that any coherent risk measure is an expectation w.r.t.~a worst-case density function $\xi P_\theta$, chosen adversarially from a suitable set of test density functions $\U(P_\theta)$, referred to as \emph{risk envelope}. Moreover, it means that any coherent risk measure is \emph{uniquely represented} by its risk envelope. Thus, in the sequel, we shall interchangeably refer to coherent risk-measures either by their explicit functional representation, or by their corresponding risk-envelope.

%   - The set U in our representation (constraints)
In this paper, we assume that the risk envelop $\U(P_\theta)$ is given in a canonical convex programming formulation, and satisfies the following conditions.
\begin{assumption}[The General Form of Risk Envelope]\label{assume:risk_envelope}
For each given policy parameter $\theta\in\mathbb R^K$, the risk envelope $\U$ of a coherent risk measure can be written as

\vspace{-0.25in}
\begin{small}
\begin{equation}\label{eq:U_as_optimization}
%\begin{split}
\mathcal U(\pprob)=\bigg\{\xi\pprob:\;g_e(\xi,\pprob)=0,\;\forall e\in\mathcal E,
\;f_i(\xi,\pprob)\leq 0,\;\forall i\in\mathcal I,
\;\sum_{\omega\in\Omega}\xi(\omega)\pprob(\omega)=1,\;\xi(\omega)\geq 0\bigg\},
%\end{split}
\end{equation}
\end{small}
\vspace{-0.225in}

where each constraint $g_e(\xi,P_\theta)$ is an affine function in $\xi$, each constraint $f_i(\xi,P_\theta)$ is a convex function in $\xi$, and there exists a strictly feasible point $\overline\xi$. $\mathcal E$ and $\mathcal I$ here denote the sets of equality and inequality constraints, respectively. Furthermore, for any given $\xi\in\mathcal B$, $f_i({\xi},p)$ and $g_e({\xi},p)$ are twice differentiable in $p$, and there exists a $M>0$ such that

%\vspace{-0.15in}
\begin{small}
\begin{equation*}
\max\left\{\max_{i\in\mathcal I}\left|\frac{ d f_i({\xi},p)}{d p(\omega)}\right|,\max_{e\in\mathcal E}\left|\frac{ d g_e({\xi},p)}{d p(\omega)}\right|\right\}\leq M,\,\forall \omega\in\Omega.
\end{equation*}
\end{small}
%\vspace{-0.2in}
\end{assumption}
Assumption~\ref{assume:risk_envelope} implies that the risk envelope $\U(P_\theta)$ is known in an \emph{explicit} form. From Theorem 6.6 of \citet{Shapiro2009}, in the case of a finite probability space, $\rho$ is a coherent risk if and only if $\U(P_\theta)$ is a convex and compact set.
%In the case of a finite probability space, by Theorem 3.18~of~\citet{rudin1991functional}, we conclude that if $\rho$ is a coherent risk, $\mathcal U(P_\theta)$ is convex and compact.\textcolor{red}{Can we simplify this?}
This justifies the affine assumption of $g_e$ and the convex assumption of $f_i$. Moreover, the additional assumption on the smoothness of the constraints holds for many popular coherent risk measures, such as the CVaR, the mean-semi-deviation, and spectral risk measures~\cite{acerbi2002spectral}.

%%%%%%%%%%%%%%%%%%%%%%%%%%%%%%%%%%%%%%%%%%%%%%%%%%%%%%%%%%%%%%%%%%%%%%%%%%%%%%%%
%%%%%%%%%%%%%%%%%%%%%%%%%%%%%%%%%%%%%%%%%%%%%%%%%%%%%%%%%%%%%%%%%%%%%%%%%%%%%%%%
%%%%%%%%%%%%%%%%%%%%%%%%%%%%%%%%%%%%%%%%%%%%%%%%%%%%%%%%%%%%%%%%%%%%%%%%%%%%%%%%

% Dynamic Markov coherent risk
\subsection{Dynamic Risk Measures}
%   - Explain motivation (Bellman equation, time consistency)

The risk measures defined above do not take into account any temporal structure that the random variable might have, such as when it is associated with the return of a trajectory in the case of MDPs. In this sense, such risk measures are called \emph{static}. {\em Dynamic} risk measures, on the other hand, explicitly take into account the temporal nature of the stochastic outcome. A primary motivation for considering such measures is the issue of \emph{time consistency}, usually defined as follows~\cite{ruszczynski2010risk}: if a certain outcome is considered less risky in all states of the world at stage $t+1$, then it should also be considered less risky at stage $t$.
%A very important class of dynamic risk measures are those that are {\em time consistent}.
Example 2.1 in~\citet{iancu2011tight} shows the importance of time consistency in the evaluation of risk in a dynamic setting. It illustrates that for multi-period decision-making, optimizing a static measure can lead to ``time-inconsistent" behavior. Similar paradoxical results could be obtained with other risk metrics; we refer the readers to~\citet{ruszczynski2010risk} and~\citet{iancu2011tight} for further insights.

\vspace{-0.1in}
\paragraph{Markov Coherent Risk Measures.}
Markov risk measures were introduced in \citet{ruszczynski2010risk} and are a useful class of dynamic time-consistent risk measures that are particularly important for our study of risk in MDPs.
%Following Eq.~\ref{eq:time-cons}, for an MDP $\mdp$,
For a $T$-length horizon and MDP $\mdp$, the Markov coherent risk measure $\rho_T(\mdp)$ is

%Equipped with the notion of time consistent risk measures, in this paper we are interested in \emph{stationary time consistent coherent risk measures}, denoted by $\rho_T$ and defined for an MDP $\mdp$ as follows:

\vspace{-10pt}
\begin{small}
\begin{equation}\label{eq:dynamic_risk_def}
\rho_T(\mdp) = C(x_0) + \gamma\rho\Bigg(C(x_1) + \ldots +\gamma\rho\Big(C(x_{T-1}) + \gamma\rho\big(C(x_T)\big)\Big)\Bigg),
\end{equation}
\end{small}
%\begin{small}
%\begin{equation}\label{eq:dynamic_risk_def}
%\rho_T(\mdp) = C(x_0,a_0) + \gamma\rho\Bigg(C(x_1,a_1) + \ldots +\gamma\rho\Big(C(x_{T-1},a_{T-1}) + \gamma\rho\big(C(x_T,a_T)\big)\Big)\Bigg),
%\end{equation}
%\end{small}
\vspace{-0.15in}

where $\rho$ is a static coherent risk measure that satisfies Assumption \ref{assume:risk_envelope} and $x_0,\dots,x_T$ is a trajectory drawn from the MDP $\mdp$ under policy $\pol$. It is important to note that in \eqref{eq:dynamic_risk_def}, each static coherent risk $\rho$ at state $x\in\mathcal X$ is induced by the transition probability $P_\theta(\cdot|x)=\sum_{a\in\mathcal A}P(x'|x,a)\mu_\theta(a|x)$. We also define $\rho_\infty(\mdp) \doteq \lim_{T \to \infty}\rho_T(\mdp)$, which is well-defined since $\gamma<1$ and the cost is bounded. We further assume that $\rho$ in \eqref{eq:dynamic_risk_def} is a \emph{Markov risk} measure, i.e.,~the evaluation of each static coherent risk measure $\rho$ is not allowed to depend on the whole past.
%Explicitly, for any $t\geq 0$ and state dependent random variable $Z(x_{t+1})\in\mathcal Z_{t+1}$, the risk evaluation is given by
%
%%\vspace{-0.2in}
%\begin{small}
%\begin{equation}
%\label{eq:representation-result}
%\rho\big(Z(x_{t+1})\big)=\max_{\xi\,:\, \xi P_\theta(\cdot |x_t)\in \U(x_t,P_\theta(\cdot |x_t))}\mathbb E_{\xi}\big[Z(x_{t+1})\big],
%\end{equation}
%\end{small}
%%\vspace{-0.2in}
%
%where we let $\U(x_t,P_\theta(\cdot |x_t))$ denote the risk-envelope \eqref{eq:U_as_optimization} with $\pprob$ replaced with $P_\theta(\cdot |x_t)$. The Markovian assumption on the risk measure $\rho_T(\mdp)$ allows us to optimize it using dynamic programming techniques. More details can be found in Section~\ref{sec:dynamic}.

%%%%%%%%%%%%%%%%%%%%%%%%%%%%%%%%%%%%%%%%%%%%%%%%%%%%%%%%%%%%%%%%%%%%%%%%%%%%%%%%
%%%%%%%%%%%%%%%%%%%%%%%%%%%%%%%%%%%%%%%%%%%%%%%%%%%%%%%%%%%%%%%%%%%%%%%%%%%%%%%%
%%%%%%%%%%%%%%%%%%%%%%%%%%%%%%%%%%%%%%%%%%%%%%%%%%%%%%%%%%%%%%%%%%%%%%%%%%%%%%%%
%%%%%%%%%%%%%%%%%%%%%%%%%%%%%%%%%%%%%%%%%%%%%%%%%%%%%%%%%%%%%%%%%%%%%%%%%%%%%%%%
%%%%%%%%%%%%%%%%%%%%%%%%%%%%%%%%%%%%%%%%%%%%%%%%%%%%%%%%%%%%%%%%%%%%%%%%%%%%%%%%

\vspace{-0.1in}
\section{Problem Formulation}\label{sec:problem_formulation}
\vspace{-0.1in}

In this paper, we are interested in solving two risk-sensitive optimization problems. Given a random variable $Z$ and a static coherent risk measure $\rho$ as defined in Section~\ref{sec:background}, the static risk problem (SRP) is given by
%
%\vspace{-3pt}
\begin{equation}\label{eq:SRP_problem}
    \min_{\param} \quad \rho(Z).
\end{equation}
%\vspace{-0.25in}
For example, in an RL setting, $Z$ may correspond to the cumulative discounted cost $Z = C(x_0) + \gamma C(x_1) + \dots +\gamma^T C(x_{T})$ of a trajectory induced by an MDP with a policy parameterized by $\theta$.

For an MDP $\mathcal{M}$ and a dynamic Markov coherent risk measure $\rho_T$ as defined by Eq.~\ref{eq:dynamic_risk_def}, the dynamic risk problem (DRP) is given by
%
%\vspace{-6pt}
\begin{equation}\label{eq:DRP_problem}
    \min_{\param} \quad \rho_\infty(\mdp).
\end{equation}
%\vspace{-0.25in}
Except for very limited cases, there is no reason to hope that neither the SRP in~\eqref{eq:SRP_problem} nor the DRP in~\eqref{eq:DRP_problem} should be tractable problems, since the dependence of the risk measure on $\theta$ may be complex and non-convex. %TODO - say this better...
In this work, we aim towards a more modest goal and search for a \emph{locally} optimal $\param$. Thus, the main problem that we are trying to solve in this paper is how to calculate the gradients of the SRP's and DRP's objective functions
%
%\vspace{-2pt}
\begin{equation*}
    \dt \rho(Z) \quad\quad \text{and} \quad\quad \dt \rho_\infty(\mdp).
\end{equation*}
%\vspace{-0.1in}
%
We are interested in non-trivial cases in which the gradients cannot be calculated analytically. In the static case, this would correspond to a non-trivial dependence of $Z$ on $\param$. For dynamic risk, we also consider cases where the state space is too large for a tractable computation. Our approach for dealing with such difficult cases is through sampling. We assume that in the static case, we may obtain i.i.d.~samples of the random variable $Z$. For the dynamic case, we assume that for each state and action $(x,a)$ of the MDP, we may obtain i.i.d.~samples of the next state $x'\sim P(\cdot|x,a)$. We show that sampling may indeed be used in both cases to devise suitable estimators for the gradients.

To finally solve the SRP and DRP problems, a gradient estimate may be plugged into a standard stochastic gradient descent (SGD) algorithm for learning a locally optimal solution to~\eqref{eq:SRP_problem} and~\eqref{eq:DRP_problem}. %TODO - say something about convergence
From the structure of the dynamic risk in Eq.~\ref{eq:dynamic_risk_def}, one may think that a gradient estimator for $\rho(Z)$ may help us to estimate the gradient $\dt \rho_\infty(\mdp)$. Indeed, we follow this idea and begin with estimating the gradient in the static risk case.
%The rest of the paper is thus structured as follows. In Section~\ref{sec:static}, we propose an estimator for $\dt \rho(Z)$,
%which is used in Section~\ref{sec:dynamic} to estimate $\dt \rho_\infty(\mdp)$.

%%%%%%%%%%%%%%%%%%%%%%%%%%%%%%%%%%%%%%%%%%%%%%%%%%%%%%%%%%%%%%%%%%%%%%%%%%%%%%%%
%%%%%%%%%%%%%%%%%%%%%%%%%%%%%%%%%%%%%%%%%%%%%%%%%%%%%%%%%%%%%%%%%%%%%%%%%%%%%%%%
%%%%%%%%%%%%%%%%%%%%%%%%%%%%%%%%%%%%%%%%%%%%%%%%%%%%%%%%%%%%%%%%%%%%%%%%%%%%%%%%
%%%%%%%%%%%%%%%%%%%%%%%%%%%%%%%%%%%%%%%%%%%%%%%%%%%%%%%%%%%%%%%%%%%%%%%%%%%%%%%%
%%%%%%%%%%%%%%%%%%%%%%%%%%%%%%%%%%%%%%%%%%%%%%%%%%%%%%%%%%%%%%%%%%%%%%%%%%%%%%%%

\vspace{-0.1in}
\section{Gradient Formula for Static Risk}\label{sec:static}
\vspace{-0.1in}
% General gradient formula using envelope theorem + proof

In this section, we consider a static coherent risk measure $\rho(Z)$ and propose sampling-based estimators for $\dt \rho(Z)$. We make the following assumption on the policy parametrization, which is standard in the policy gradient literature~\citep{MarTsi98}.
\begin{assumption}\label{ass:LR_well_behaved}
The likelihood ratio $\dt \log P(\omega)$ is well-defined and bounded for all $\omega \!\in \!\Omega$.
\end{assumption}
%
%\vspace{-2pt}
Moreover, our approach implicitly assumes that given some $\omega\in\Omega$, $\dt \log P(\omega)$ may be easily calculated. This is also a standard requirement for policy gradient algorithms~\cite{MarTsi98} and is satisfied in various applications such as queueing systems, inventory management, and financial engineering (see, e.g.,~the survey by Fu~\citealt{Fu2006gradients}).

Using Theorem~\ref{thm:rep} and Assumption~\ref{assume:risk_envelope}, for each $\theta$, we have that $\rho(Z)$ is the solution to the convex optimization problem~\eqref{eq:coherent_as_optimization} (for that value of $\theta$). The Lagrangian function of~\eqref{eq:coherent_as_optimization}, denoted by $L_{\theta}(\xi,\lambda^{\mathcal P},\lambda^{\mathcal E},\lambda^{\mathcal I})$, may be written as

\vspace{-0.2in}
\begin{small}
\begin{equation}\label{eq:Lagrangian}
%\begin{split}
L_{\theta}(\xi,\lambda^{\mathcal P}\!\!,\lambda^{\mathcal E}\!\!,\lambda^{\mathcal I})
\!=\!\!\sum_{\omega \in \Omega} \!\xi(\omega) P_\theta(\omega) Z(\omega)-\lambda^{\mathcal P}\!\left(\sum_{\omega \in \Omega}\xi(\omega)P_\theta(\omega)\!-\!1\!\right)
-\sum_{e\in\mathcal E}\lambda^{\mathcal E}(e) g_e(\xi,\!P_\theta)-\sum_{i\in\mathcal I}\lambda^{\mathcal I}(i) f_i(\xi,\!P_\theta).
%\end{split}
\end{equation}
\end{small}
\vspace{-0.2in}

The convexity of~\eqref{eq:coherent_as_optimization} and its strict feasibility due to Assumption~\ref{assume:risk_envelope} implies that $L_{\theta}(\xi,\lambda^{\mathcal P},\lambda^{\mathcal E},\lambda^{\mathcal I})$ has a non-empty set of saddle points $\spset$. The next theorem presents a formula for the gradient $\dt \rho(Z)$. As we shall subsequently show, this formula is particularly convenient for devising sampling based estimators for $\dt \rho(Z)$.
\begin{theorem}\label{thm:static_gradient}
Let Assumptions~\ref{assume:risk_envelope} and~\ref{ass:LR_well_behaved} hold. For any saddle point $(\xi^*_{\theta},\lambda^{*,\mathcal P}_{\theta},\lambda^{*,\mathcal E}_{\theta},\lambda^{*,\mathcal I}_{\theta}) \in \spset$ of~\eqref{eq:Lagrangian}, we have
\begin{equation*}
%\begin{split}
  \dt{\rho(Z)} = \ExpW{\dt \log P(\omega) (Z - \lambda^{*,\mathcal P}_{\theta})}{\xi^*_{\theta}}
     - \sum_{e\in\mathcal E} \lambda^{*,\mathcal E}_{\theta}(e) \dt{g_e(\xi^*_{\theta};P_\theta)}
     -\sum_{i\in\mathcal I} \lambda^{*,\mathcal I}_{\theta}(i) \dt{f_i(\xi^*_{\theta};P_\theta)}.
%\end{split}
\end{equation*}
\end{theorem}
%
%\vspace{-10pt}
The proof of this theorem, given in the supplementary material, involves an application of the Envelope theorem~\cite{milgrom2002envelope} and a standard `likelihood-ratio' trick. We now demonstrate the utility of Theorem \ref{thm:static_gradient} with several examples in which we show that it generalizes previously known results, and also enables deriving new useful gradient formulas.
%The full details for deriving these results are in the supplementary material.
% submitted along with the paper.

%%%%%%%%%%%%%%%%%%%%%%%%%%%%%%%%%%%%%%%%%%%%%%%%%%%%%%%%%%%%%%%%%%%%%%%%%%%%%%%%
%%%%%%%%%%%%%%%%%%%%%%%%%%%%%%%%%%%%%%%%%%%%%%%%%%%%%%%%%%%%%%%%%%%%%%%%%%%%%%%%
%%%%%%%%%%%%%%%%%%%%%%%%%%%%%%%%%%%%%%%%%%%%%%%%%%%%%%%%%%%%%%%%%%%%%%%%%%%%%%%%

% Example 1 (analytic) - CVaR (just gradient formula + relation to previous results)
\vspace{-0.1in}
\subsection{Example 1: CVaR}
\vspace{-0.05in}

The CVaR at level $\alpha\in[0,1]$ of a random variable $Z$, denoted by $\cvar(Z;\alpha)$, is a very popular coherent risk measure~\citep{rockafellar2000optimization}, defined as
\begin{equation*}
\cvar(Z;\alpha) \doteq \inf_{t\in\reals} \big\{ t + \alpha^{-1} \Exp{(Z-t)_{+}}\big\}.
\end{equation*}
When $Z$ is continuous, $\cvar(Z;\alpha)$ is well-known to be the mean of the $\alpha$-tail distribution of $Z$, $\Exp{ \left. Z \right| Z > \quan }$, where $\quan$ is a $(1-\alpha)$-quantile of $Z$. Thus, selecting a small $\alpha$ makes CVaR particularly sensitive to rare, but very high costs.

The risk envelope for CVaR is known to be~\citep{Shapiro2009}
%\begin{equation*}
$
\U = \big\{ \xi \pprob: \xi(\omega) \in [0,\alpha^{-1}],\quad \sum_{\omega\in\Omega} \xi(\omega)P_\theta(\omega)=1\big\}.
$
%\end{equation*}
Furthermore, \citet{Shapiro2009} show that the saddle points of \eqref{eq:Lagrangian} satisfy $\xi^*_{\theta}(\omega)=\alpha^{-1}$ when $Z(\omega)>\lambda^{*,\mathcal P}_{\theta}$, and $\xi^*_{\theta}(\omega)=0$ when $Z(\omega)<\lambda^{*,\mathcal P}_{\theta}$, where $\lambda^{*,\mathcal P}_{\theta}$ is any $(1-\alpha)$-quantile of $Z$. Plugging this result into Theorem \ref{thm:static_gradient}, we can easily show that
\begin{equation*}
    \dt\cvar(Z;\alpha) = \Exp{\left. \dt \log P(\omega) (Z - \quan)\right| Z(\omega) > \quan}.
\end{equation*}
This formula was recently proved in~\citet{tamar2015optimizing} for the case of continuous distributions by an explicit calculation of the conditional expectation, and under several additional smoothness assumptions. Here we show that it holds regardless of these assumptions and in the discrete case as well. Our proof is also considerably simpler.

%%%%%%%%%%%%%%%%%%%%%%%%%%%%%%%%%%%%%%%%%%%%%%%%%%%%%%%%%%%%%%%%%%%%%%%%%%%%%%%%
%%%%%%%%%%%%%%%%%%%%%%%%%%%%%%%%%%%%%%%%%%%%%%%%%%%%%%%%%%%%%%%%%%%%%%%%%%%%%%%%
%%%%%%%%%%%%%%%%%%%%%%%%%%%%%%%%%%%%%%%%%%%%%%%%%%%%%%%%%%%%%%%%%%%%%%%%%%%%%%%%

% Example 2 (analytic) - Mean-semideviation (just gradient formula)
\vspace{-0.1in}
\subsection{Example 2: Mean-Semideviation}
\vspace{-0.1in}
The semi-deviation of a random variable $Z$ is defined as $\SD[Z] \doteq \left( \Exp{(Z - \Exp{Z})_{+}^{2}} \right)^{1 / 2}$. The semi-deviation captures the variation of the cost only \emph{above its mean}, and is an appealing alternative to the standard deviation, which does not distinguish between the variability of upside and downside deviations.
For some $\alpha \in[0,1]$, the \emph{mean-semideviation} risk measure is defined as $\msd(Z;\alpha) \doteq \Exp{Z} + \alpha \SD[Z]$, and is a coherent risk measure~\citep{Shapiro2009}. We have the following result:
\begin{proposition}\label{prop:msd_grad}
Under Assumption~\ref{ass:LR_well_behaved}, with $\dt{\Exp{Z}} = \Exp{\dt \log P(\omega) Z}$, we have
\begin{equation*}
%\begin{split}
     \dt{\msd(Z;\alpha)} = \dt{\Exp{Z}} +
     \! \frac{\alpha \Exp{(Z \!-\! \Exp{Z})_{+} \!\left(\dt \log P(\omega) (Z \!-\! \Exp{Z}) \!-\! \dt{\Exp{Z}}\right)}}{\SD(Z)}.
%\end{split}
\end{equation*}
\end{proposition}
This proposition can be used to devise a sampling based estimator for $\dt{\msd(Z;\alpha)}$ by replacing all the expectations with sample averages. The algorithm along with the proof of the proposition are in the supplementary material. In Section \ref{sec:experiment} we provide a numerical illustration of optimization with a mean-semideviation objective.

%%%%%%%%%%%%%%%%%%%%%%%%%%%%%%%%%%%%%%%%%%%%%%%%%%%%%%%%%%%%%%%%%%%%%%%%%%%%%%%%
%%%%%%%%%%%%%%%%%%%%%%%%%%%%%%%%%%%%%%%%%%%%%%%%%%%%%%%%%%%%%%%%%%%%%%%%%%%%%%%%
%%%%%%%%%%%%%%%%%%%%%%%%%%%%%%%%%%%%%%%%%%%%%%%%%%%%%%%%%%%%%%%%%%%%%%%%%%%%%%%%

\vspace{-0.1in}
\subsection{General Gradient Estimation Algorithm}
\vspace{-0.05in}
% General sampling procedure
%   - Give algorithm
%   - Consistency proof (?)
%   - Example - CVaR
%   - Example - Mean-semideviation (?)

In the two previous examples, we obtained a gradient formula by \emph{analytically} calculating the Lagrangian saddle point~\eqref{eq:Lagrangian} and plugging it into the formula of Theorem~\ref{thm:static_gradient}. We now consider a general coherent risk $\rho(Z)$ for which, in contrast to the CVaR and mean-semideviation cases, the Lagrangian saddle-point is not known analytically. \emph{We only assume that we know the structure of the risk-envelope} as given by~\eqref{eq:U_as_optimization}.
We show that in this case, $\dt \rho(Z)$ may be estimated using a \emph{sample average approximation} (SAA;~\citealt{Shapiro2009}) of the formula in Theorem~\ref{thm:static_gradient}.

Assume that we are given $N$ i.i.d.~samples $\omega_i \sim \pprob$, $i=1,\dots,N$, and let $\pemp(\omega) \doteq \frac{1}{N}\sum_{i=1}^N \ind{ \omega_i = \omega }$ denote the corresponding empirical distribution. Also, let the \emph{sample risk envelope} $\mathcal U(\pemp)$ be defined according to Eq.~\ref{eq:U_as_optimization} with $\pprob$ replaced by $\pemp$. Consider the following SAA version of the optimization in Eq.~\ref{eq:coherent_as_optimization}:
\begin{equation}\label{eq:SAA_coherent}
\rho_N(Z) = \max_{\xi:\xi \pemp \in \mathcal U(\pemp)} \sum_{i \in 1,\dots,N} \pemp(\omega_i) \xi(\omega_i) Z(\omega_i).
\end{equation}
Note that~\eqref{eq:SAA_coherent} defines a convex optimization problem with $\mathcal O(N)$ variables and constraints. In the following, we assume that a solution to~\eqref{eq:SAA_coherent} may be computed efficiently using standard convex programming tools such as interior point methods~\cite{boyd2009convex}. Let $\xi^*_{\theta;N}$ denote a solution to~\eqref{eq:SAA_coherent} and $\lambda^{*,\mathcal P}_{\theta;N},\lambda^{*,\mathcal E}_{\theta;N},\lambda^{*,\mathcal I}_{\theta;N}$ denote the corresponding KKT multipliers, which can be obtained from the convex programming algorithm~\cite{boyd2009convex}. We propose the following estimator for the gradient-based on Theorem~\ref{thm:static_gradient}:
\vspace{-0.05in}
\begin{align}\label{eq:SAA_gradient}
  \dtN{\rho(Z)} &= \sum_{i = 1}^{N} \pemp(\omega_i) \xi^*_{\theta;N}(\omega_i) \dt \log P(\omega_i) (Z(\omega_i) - \lambda^{*,\mathcal P}_{\theta;N}) \\
    & - \sum_{e\in\mathcal E} \lambda^{*,\mathcal E}_{\theta;N}(e) \dt{g_e(\xi^*_{\theta;N};P_{\theta;N})}\nonumber
     -\sum_{i\in\mathcal I} \lambda^{*,\mathcal I}_{\theta;N}(i) \dt{f_i(\xi^*_{\theta;N};P_{\theta;N})}.
\end{align}
Thus, our gradient estimation algorithm is a two-step procedure involving \emph{both sampling and convex programming}. In the following, we show that under some conditions on the set $\mathcal U(\pprob)$, $\dtN{\rho(Z)}$ is a consistent estimator of $\dt{\rho(Z)}$. The proof has been reported in the supplementary material.
\begin{proposition}\label{prop:consistent}
Let Assumptions~\ref{assume:risk_envelope} and~\ref{ass:LR_well_behaved} hold. Suppose there exists a compact set $C = C_\xi \times C_\lambda$ such that:
(I) The set of Lagrangian saddle points $\spset \subset C$ is non-empty and bounded.
(II) The functions $f_e(\xi,P_\theta)$ for all $e\in\mathcal E$ and $f_i(\xi,P_\theta)$ for all $i\in\mathcal I$ are finite-valued and continuous (in $\xi$) on $C_\xi$.
(III) For $N$ large enough, the set $\saaspset$ is non-empty and $\saaspset \subset C$ w.p.~1.
Further assume that:
(IV) If $\xi_N\pemp \in \mathcal U(\pemp)$ and $\xi_N$ converges w.p.~1 to a point $\xi$, then $\xi\pprob\in\mathcal U(\pprob)$.
We then have that $\lim_{N\to \infty} \rho_N(Z) = \rho(Z) $ and $\lim_{N\to \infty} \dtN{\rho(Z)} = \dt{\rho(Z)} $ w.p.~1.
\end{proposition}
The set of assumptions for Proposition~\ref{prop:consistent} is large, but rather mild. Note that (I) is implied by the Slater condition of Assumption~\ref{assume:risk_envelope}. For satisfying (III), we need that the risk be well-defined for every empirical distribution, which is a natural requirement. Since $P_{\theta;N}$ always converges to $P_{\theta}$ uniformly on $\Omega$, (IV) essentially requires smoothness of the constraints. We remark that in particular, constraints (I) to (IV) are satisfied for the popular CVaR, mean-semideviation, and spectral risk measures.

% Explain that this result motivates our approach for static+dynamic : sampling (or analytic calculation) to solve the optimization and find u, and then sampling according to u to calculate the gradient.
To summarize this section, we have seen that by exploiting the special structure of coherent risk measures in Theorem~\ref{thm:rep} and by the envelope-theorem style result of Theorem~\ref{thm:static_gradient}, we were able to derive sampling-based, likelihood-ratio style algorithms for estimating the policy gradient $\dt{\rho(Z)}$ of coherent static risk measures. The gradient estimation algorithms developed here for static risk measures will be used as a sub-routine in our subsequent treatment of dynamic risk measures.

%%%%%%%%%%%%%%%%%%%%%%%%%%%%%%%%%%%%%%%%%%%%%%%%%%%%%%%%%%%%%%%%%%%%%%%%%%%%%%%%
%%%%%%%%%%%%%%%%%%%%%%%%%%%%%%%%%%%%%%%%%%%%%%%%%%%%%%%%%%%%%%%%%%%%%%%%%%%%%%%%
%%%%%%%%%%%%%%%%%%%%%%%%%%%%%%%%%%%%%%%%%%%%%%%%%%%%%%%%%%%%%%%%%%%%%%%%%%%%%%%%
%%%%%%%%%%%%%%%%%%%%%%%%%%%%%%%%%%%%%%%%%%%%%%%%%%%%%%%%%%%%%%%%%%%%%%%%%%%%%%%%
%%%%%%%%%%%%%%%%%%%%%%%%%%%%%%%%%%%%%%%%%%%%%%%%%%%%%%%%%%%%%%%%%%%%%%%%%%%%%%%%

%\vspace{-0.1in}
\section{Gradient Formula for Dynamic Risk}
\label{sec:dynamic}
%\vspace{-0.1in}

In this section, we derive a new formula for the gradient of the Markov coherent dynamic risk measure, $\dt \rho_\infty(\mdp)$. Our approach is based on combining the static gradient formula of Theorem~\ref{thm:static_gradient}, with a dynamic-programming decomposition of $\rho_\infty(\mdp)$.

The risk-sensitive \emph{value-function} for an MDP $\mdp$ under the policy $\theta$ is defined as $V_\theta(x)=\rho_\infty(\mdp | x_0 = x)$, where with a slight abuse of notation, $\rho_\infty(\mdp | x_0 = x)$ denotes the Markov-coherent dynamic risk in~\eqref{eq:dynamic_risk_def} when the initial state $x_0$ is $x$. It is shown in~\citep{ruszczynski2010risk} that due to the structure of the Markov dynamic risk $\rho_\infty(\mdp)$, the value function is the unique solution to the \emph{risk-sensitive Bellman equation}
\begin{equation}\label{eq:T}
V_\theta(x) = C(x) + \gamma\max_{\xi P_\theta(\cdot |x)\in \U(x,P_\theta(\cdot |x))}\mathbb E_{\xi}[V_\theta(x')],
\end{equation}
%
%where $C_\theta(x)=\sum_{a\in\mathcal A}C(x,a)\mu_\theta(a|x)$ \textcolor{red}{(is this correct??)} is the stage-wise cost function induced by policy $\mu_\theta$.
where the expectation is taken over the next state transition. Note that by definition, we have $\rho_\infty(\mdp) = V_\theta(x_0)$, and thus, $\dt \rho_\infty(\mdp) = \dt V_\theta(x_0)$.

We now develop a formula for $\dt V_\theta(x)$; this formula extends the well-known ``policy gradient theorem"~\cite{sutton_policy_2000,konda2000actor}, developed for the expected return, to Markov-coherent dynamic risk measures. We make a standard assumption, analogous to Assumption \ref{ass:LR_well_behaved} of the static case.
\begin{assumption}\label{assume:ll_ratio_bounded}
The likelihood ratio $\nabla_\theta\log\mu_\theta(a|x)$ is well-defined and bounded for all $x\in\mathcal X$ and $a\in\mathcal A$.
\end{assumption}
\vspace{-0.1in}
For each state $x\in\mathcal X$, let $(\xi^*_{\theta,x},\lambda^{*,\mathcal P}_{\theta,x},\lambda^{*,\mathcal E}_{\theta,x},\lambda^{*,\mathcal I}_{\theta,x})$ denote a saddle point of~\eqref{eq:Lagrangian}, corresponding to the state $x$, with $P_\theta(\cdot |x)$ replacing $P_\theta$ in \eqref{eq:Lagrangian} and $V_\theta$ replacing $Z$. The next theorem presents a formula for $\dt V_\theta(x)$; the proof is in the supplementary material.

\begin{theorem}\label{thm:dynamic_risk}
Under Assumptions~\ref{assume:risk_envelope} and~\ref{assume:ll_ratio_bounded}, we have
\begin{equation*}
\nabla V_\theta(x) = \mathbb E_{\xi^*_\theta} \left[\left.\sum_{t=0}^{\infty}\gamma^t\nabla_\theta\log\mu_\theta(a_t|x_t)h_\theta(x_t,a_t)\right|  x_0=x\right],
\end{equation*}
where $\mathbb E_{\xi^*_\theta}[\cdot]$ denotes the expectation w.r.t.~trajectories generated by the Markov chain with transition probabilities $P_\theta(\cdot|x)\xi_{\theta,x}^*(\cdot)$, and the stage-wise cost function $h_\theta(x,a)$ is defined as
\begin{small}
\begin{equation*}
h_\theta(x,a) \!=\! C(x) + \!\sum_{x'  \in \mathcal X} \!\!P(x'|x,a)\xi^*_{\theta,x}(x')\!\!\left[\gamma V_\theta(x')\!-\!{\lambda}^{*,\mathcal P}_{\theta,x}
\!-\!\sum_{i\in\mathcal I} \!{\lambda}^{*,\mathcal I}_{\theta,x}(i)\frac{ d f_i(\xi^*_{\theta,x},p)}{d p(x')} \!-\! \sum_{e\in\mathcal E}\!{\lambda}^{*,\mathcal E}_{\theta,x}(e) \frac{ d g_e(\xi^*_{\theta,x},p)}{d p(x')}\right]\!\!.
\label{eq:h}
\end{equation*}
\end{small}
\end{theorem}
Theorem \ref{thm:dynamic_risk} may be used to develop an \emph{actor-critic} style~\cite{sutton_policy_2000,konda2000actor} sampling-based algorithm for solving the DRP problem \eqref{eq:DRP_problem}, composed of two interleaved procedures:

\textbf{Critic:} For a given policy $\theta$, calculate the risk-sensitive value function $V_\theta$, and \\
\textbf{Actor:} Using the critic's $V_\theta$ and Theorem \ref{thm:dynamic_risk}, estimate $\dt \rho_\infty(\mdp)$ and update $\theta$.

Space limitation restricts us from specifying the full details of our actor-critic algorithm and its analysis. In the following, we highlight only the key ideas and results. For the full details, we refer the reader to the full paper version, provided in the supplementary material.

For the critic, the main challenge is calculating the value function when the state space $\mathcal X$ is large and dynamic programming cannot be applied due to the `curse of dimensionality'. To overcome this, we exploit the fact that $V_\theta$ is equivalent to the value function in a robust MDP~\cite{osogami2012robustness} and modify a recent algorithm in~\citet{tamar2014robust} to estimate it using function approximation.

For the actor, the main challenge is that in order to estimate the gradient using Thm.~\ref{thm:dynamic_risk}, we need to sample from an MDP with $\xi_{\theta}^*$-weighted transitions. Also, $h_\theta(x,a)$ involves an expectation for each $s$ and $a$. Therefore, we propose a \emph{two-phase sampling procedure} to estimate $\nabla V_\theta$ in which we first use the critic's estimate of $V_\theta$ to derive $\xi_{\theta}^*$, and sample a trajectory from an MDP with $\xi_{\theta}^*$-weighted transitions. For each state in the trajectory, we then sample several next states to estimate $h_\theta(x,a)$.

The convergence analysis of the actor-critic algorithm and the gradient error incurred from function approximation of $V_\theta$ are reported in the supplementary material.

%%%%%%%%%%%%%%%%%%%%%%%%%%%%%%%%%%%%%%%%%%%%%%%%%%%%%%%%%%%%%%%%%%%%%%%%%%%%%%%%%
%%%%%%%%%%%%%%%%%%%%%%%%%%%%%%%%%%%%%%%%%%%%%%%%%%%%%%%%%%%%%%%%%%%%%%%%%%%%%%%%%
%%%%%%%%%%%%%%%%%%%%%%%%%%%%%%%%%%%%%%%%%%%%%%%%%%%%%%%%%%%%%%%%%%%%%%%%%%%%%%%%%
%%%%%%%%%%%%%%%%%%%%%%%%%%%%%%%%%%%%%%%%%%%%%%%%%%%%%%%%%%%%%%%%%%%%%%%%%%%%%%%%%
%%%%%%%%%%%%%%%%%%%%%%%%%%%%%%%%%%%%%%%%%%%%%%%%%%%%%%%%%%%%%%%%%%%%%%%%%%%%%%%%%

%\vspace{-0.12in}
\section{Numerical Illustration}\label{sec:experiment}
\vspace{-0.1in}

In this section, we illustrate our approach with a numerical example. The purpose of this illustration is to emphasize the importance of \emph{flexibility} in designing risk criteria for selecting an \emph{appropriate} risk-measure -- such that suits both the user's risk preference \emph{and} the problem-specific properties.

We consider a trading agent that can invest in one of three assets (see Figure~\ref{fig:1} for their distributions). The returns of the first two assets, $A1$ and $A2$, are normally distributed: $A1\sim \mathcal N (1,1)$ and $A2\sim \mathcal N (4,6)$. The return of the third asset $A3$ has a Pareto distribution: $f(z) = \frac{\alpha}{z^{\alpha+1}}~\forall z>1$, with $\alpha = 1.5$. The mean of the return from $A3$ is 3 and its variance is infinite; such heavy-tailed distributions are widely used in financial modeling \cite{rachev2000stable}. The agent selects an action randomly, with probability $P(A_i)\propto \exp (\theta_i)$, where $\theta\in \mathbb R^3$ is the policy parameter. We trained three different policies $\pi_1$, $\pi_2$, and $\pi_3$. Policy $\pi_1$ is risk-neutral, i.e.,~$\max_\theta \Exp{Z}$, and it was trained using standard policy gradient \cite{MarTsi98}. Policy $\pi_2$ is risk-averse and had a mean-semideviation objective $\max_\theta \Exp{Z} - \SD[Z]$, and was trained using the algorithm in Section~\ref{sec:static}. Policy $\pi_3$ is also risk-averse, with a mean-standard-deviation objective, as proposed in~\cite{tamar2012policy,prashanth2013actor}, $\max_\theta \Exp{Z} - \sqrt{\textrm{Var}[Z]}$, and was trained using the algorithm of~\cite{tamar2012policy}. For each of these policies, Figure~\ref{fig:1} shows the probability of selecting each asset vs.~training iterations. Although $A2$ has the highest mean return, the risk-averse policy $\pi_2$ chooses $A3$, since it has a lower downside, as expected. However, because of the heavy upper-tail of $A3$, policy $\pi_3$ opted to choose $A1$ instead. This is counter-intuitive as a rational investor should not avert high returns. In fact, in this case $A3$ stochastically dominates $A1$~\cite{hadar1969rules}.

\begin{figure}
\centering
\includegraphics[width=\textwidth]{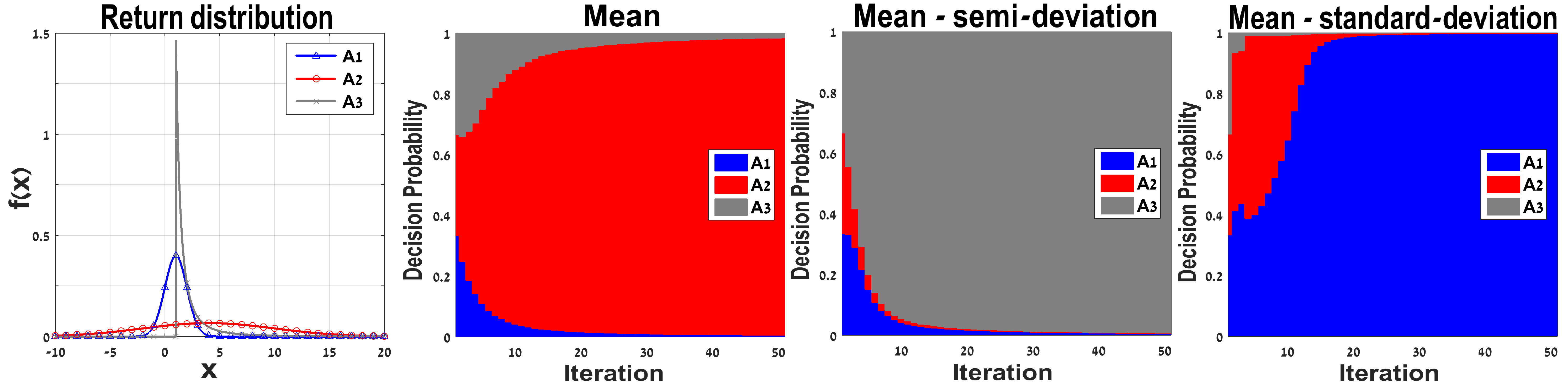}
  \caption{Numerical illustration - selection between 3 assets. A: Probability density of asset return. B,C,D: Bar plots of the probability of selecting each asset vs.~training iterations, for policies $\pi_1$, $\pi_2$, and $\pi_3$, respectively. At each iteration, 10,000 samples were used for gradient estimation.}\label{fig:1}
  \vspace{-0.1in}
\end{figure}

%%%%%%%%%%%%%%%%%%%%%%%%%%%%%%%%%%%%%%%%%%%%%%%%%%%%%%%%%%%%%%%%%%%%%%%%%%%%%%%%
%%%%%%%%%%%%%%%%%%%%%%%%%%%%%%%%%%%%%%%%%%%%%%%%%%%%%%%%%%%%%%%%%%%%%%%%%%%%%%%%
%%%%%%%%%%%%%%%%%%%%%%%%%%%%%%%%%%%%%%%%%%%%%%%%%%%%%%%%%%%%%%%%%%%%%%%%%%%%%%%%
%%%%%%%%%%%%%%%%%%%%%%%%%%%%%%%%%%%%%%%%%%%%%%%%%%%%%%%%%%%%%%%%%%%%%%%%%%%%%%%%
%%%%%%%%%%%%%%%%%%%%%%%%%%%%%%%%%%%%%%%%%%%%%%%%%%%%%%%%%%%%%%%%%%%%%%%%%%%%%%%%

\vspace{-0.05in}
\section{Conclusion}
\vspace{-0.05in}

We presented algorithms for estimating the gradient of both static and dynamic coherent risk measures using two new policy gradient style formulas that combine sampling with convex programming. Thereby, our approach extends risk-sensitive RL to the whole class of coherent risk measures, and generalizes several recent studies that focused on specific risk measures.

On the technical side, an important future direction is to improve the convergence rate of gradient estimates using importance sampling methods. This is especially important for risk criteria that are sensitive to rare events, such as the CVaR \cite{bardou2009computing}.

From a more conceptual point of view, the coherent-risk framework explored in this work provides the decision maker  with
\emph{flexibility} in designing risk preference. As our numerical example shows, such flexibility is important for selecting appropriate \emph{problem-specific} risk measures for managing the cost variability. However, we believe that our approach has much more potential than that.

In almost every real-world application, uncertainty emanates from  stochastic dynamics, but also, and perhaps more importantly, from modeling errors (model uncertainty). A prudent policy should protect against \emph{both} types of uncertainties. The representation duality of coherent-risk (Theorem \ref{thm:rep}), naturally relates the risk to model uncertainty. In \cite{osogami2012robustness}, a similar connection was made between model-uncertainty in MDPs and dynamic Markov coherent risk. We believe that by carefully shaping the risk-criterion, the decision maker may be able to take uncertainty into account in a  \emph{broad} sense.
Designing a principled procedure for such \emph{risk-shaping} is not trivial, and is beyond the scope of this paper. However, we believe that there is much potential to risk shaping as it may be the key for handling model misspecification in dynamic decision making.

\begin{small}
\bibliography{CoherentRiskArXiv15}
\bibliographystyle{plain}
\end{small}
%%%%%%%%%%%%%%%%%%%%%%%%%%%%%%%%%%%%%%%%%%%%%%%%%%%%%%%%%%%%%%%%%%%%%%%%%%%%%%%%
%%%%%%%%%%%%%%%%%%%%%%%%%%%%%%%%%%%%%%%%%%%%%%%%%%%%%%%%%%%%%%%%%%%%%%%%%%%%%%%%
%%%%%%%%%%%%%%%%%%%%%%%%%%%%%%%%%%%%%%%%%%%%%%%%%%%%%%%%%%%%%%%%%%%%%%%%%%%%%%%%
%%%%%%%%%%%%%%%%%%%%%%%%%%%%%%%%%%%%%%%%%%%%%%%%%%%%%%%%%%%%%%%%%%%%%%%%%%%%%%%%
%%%%%%%%%%%%%%%%%%%%%%%%%%%%%%%%%%%%%%%%%%%%%%%%%%%%%%%%%%%%%%%%%%%%%%%%%%%%%%%%
%%%%%%%%%%%%%%%%%%%%%%%%%%%%%%%%%%%%%%%%%%%%%%%%%%%%%%%%%%%%%%%%%%%%%%%%%%%%%%%%
%%%%%%%%%%%%%%%%%%%%%%%%%%%%%%%%%%%%%%%%%%%%%%%%%%%%%%%%%%%%%%%%%%%%%%%%%%%%%%%%
%%%%%%%%%%%%%%%%%%%%%%%%%%%%%%%%%%%%%%%%%%%%%%%%%%%%%%%%%%%%%%%%%%%%%%%%%%%%%%%%
%%%%%%%%%%%%%%%%%%%%%%%%%%%%%%%%%%%%%%%%%%%%%%%%%%%%%%%%%%%%%%%%%%%%%%%%%%%%%%%%
%%%%%%%%%%%%%%%%%%%%%%%%%%%%%%%%%%%%%%%%%%%%%%%%%%%%%%%%%%%%%%%%%%%%%%%%%%%%%%%%
\newpage
\appendix
\onecolumn

%%%%%%%%%%%%%%%%%%%%%%%%%%%%%%%%%%%%%%%%%%%%%%%%%%%%%%%%%%%%%%%%%%%%%%%%%%%%%%%%
%%%%%%%%%%%%%%%%%%%%%%%%%%%%%%%%%%%%%%%%%%%%%%%%%%%%%%%%%%%%%%%%%%%%%%%%%%%%%%%%
%%%%%%%%%%%%%%%%%%%%%%%%%%%%%%%%%%%%%%%%%%%%%%%%%%%%%%%%%%%%%%%%%%%%%%%%%%%%%%%%
%%%%%%%%%%%%%%%%%%%%%%%%%%%%%%%%%%%%%%%%%%%%%%%%%%%%%%%%%%%%%%%%%%%%%%%%%%%%%%%%
%%%%%%%%%%%%%%%%%%%%%%%%%%%%%%%%%%%%%%%%%%%%%%%%%%%%%%%%%%%%%%%%%%%%%%%%%%%%%%%%
\section{Proof of Theorem \ref{thm:static_gradient}}

First note from Assumption~\ref{assume:risk_envelope} that
\begin{description}
\item [(i)] Slater's condition holds in the primal optimization problem~\eqref{eq:coherent_as_optimization},
\item [(ii)] $L_{\theta}(\xi,\lambda^{\mathcal P},\lambda^{\mathcal E},\lambda^{\mathcal I})$ is convex in $\xi$ and concave in $(\lambda^{\mathcal P},\lambda^{\mathcal E},\lambda^{\mathcal I})$.
\end{description}
Thus by the duality result in convex optimization \cite{boyd2009convex}, the above conditions imply strong duality and we have $\rho(Z) = \max_{\xi\geq 0}\min_{\lambda^{\mathcal P},\lambda^{\mathcal I}\geq 0,\lambda^{\mathcal E}} L_{\theta}(\xi,\lambda^{\mathcal P},\lambda^{\mathcal E},\lambda^{\mathcal I})= \min_{\lambda^{\mathcal P},\lambda^{\mathcal I}\geq 0,\lambda^{\mathcal E}}\max_{\xi\geq 0} L_{\theta}(\xi,\lambda^{\mathcal P},\lambda^{\mathcal E},\lambda^{\mathcal I})$. From Assumption~\ref{assume:risk_envelope}, one can also see that the family of functions $\{L_{\theta}(\xi,\lambda^{\mathcal P},\lambda^{\mathcal E},\lambda^{\mathcal I})\}_{(\xi,\lambda^{\mathcal P},\lambda^{\mathcal E},\lambda^{\mathcal I})\in\reals^{|\Omega|}\times\reals\times\reals^{|\mathcal E|}\times\reals^{|\mathcal I|}}$ is equi-differentiable in $\theta$, $ L_{\theta}(\xi,\lambda^{\mathcal P},\lambda^{\mathcal E},\lambda^{\mathcal I})$ is Lipschitz, as a result, an absolutely continuous function in $\theta$, and thus, $\nabla_\theta L_{\theta}(\xi,\lambda^{\mathcal P},\lambda^{\mathcal E},\lambda^{\mathcal I})$ is continuous and bounded at each $(\xi,\lambda^{\mathcal P},\lambda^{\mathcal E},\lambda^{\mathcal I})$. Then for every selection of saddle point $(\xi^*_{\theta},\lambda^{*,\mathcal P}_{\theta},\lambda^{*,\mathcal E}_{\theta},\lambda^{*,\mathcal I}_{\theta}) \in \spset$ of~\eqref{eq:Lagrangian}, using the Envelop theorem for saddle-point problems (see Theorem~4~of \citealt{milgrom2002envelope}), we have
\begin{equation}\label{eq:envelope_theorem}
\nabla_\theta\max_{\xi\geq 0}\;\min_{\lambda^{\mathcal P},\lambda^{\mathcal I}\geq 0,\lambda^{\mathcal E}} L_{\theta}(\xi,\lambda^{\mathcal P},\lambda^{\mathcal E},\lambda^{\mathcal I}) = \nabla_\theta L_{\theta}(\xi,\lambda^{\mathcal P},\lambda^{\mathcal E},\lambda^{\mathcal I})\!\!\mid_{(\xi^*_{\theta},\lambda^{*,\mathcal P}_{\theta},\lambda^{*,\mathcal E}_{\theta},\lambda^{*,\mathcal I}_{\theta})}.
\end{equation}
The result follows by writing the gradient in~\eqref{eq:envelope_theorem} explicitly, and using the likelihood-ratio trick:
\begin{equation*}
   \sum_{\omega \in \Omega} \!\! \xi(\omega) \dt P_\theta(\omega) Z(\omega)\!-\!\lambda^{\mathcal P}\sum_{\omega \in \Omega}\xi(\omega)\dt P_\theta(\omega) = \sum_{\omega \in \Omega} \!\! \xi(\omega) P(\omega) \dt \log P(\omega) \left(Z(\omega)\!-\!\lambda^{\mathcal P}\right),
\end{equation*}
where the last equality is justified by Assumption~\ref{ass:LR_well_behaved}.

%%%%%%%%%%%%%%%%%%%%%%%%%%%%%%%%%%%%%%%%%%%%%%%%%%%%%%%%%%%%%%%%%%%%%%%%%%%%%%%%
%%%%%%%%%%%%%%%%%%%%%%%%%%%%%%%%%%%%%%%%%%%%%%%%%%%%%%%%%%%%%%%%%%%%%%%%%%%%%%%%
%%%%%%%%%%%%%%%%%%%%%%%%%%%%%%%%%%%%%%%%%%%%%%%%%%%%%%%%%%%%%%%%%%%%%%%%%%%%%%%%
%%%%%%%%%%%%%%%%%%%%%%%%%%%%%%%%%%%%%%%%%%%%%%%%%%%%%%%%%%%%%%%%%%%%%%%%%%%%%%%%
%%%%%%%%%%%%%%%%%%%%%%%%%%%%%%%%%%%%%%%%%%%%%%%%%%%%%%%%%%%%%%%%%%%%%%%%%%%%%%%%

\section{Gradient Results for Static Mean-Semideviation}\label{sec:MSD_supp}

In this section we consider the mean-semideviation risk measure, defined as follows:
\begin{equation}\label{eq:mean-sd_def}
    \msd(Z) = \Exp{Z} + c\left( \Exp{(Z - \Exp{Z})_{+}^{2}} \right)^{1 / 2},
\end{equation}
Following the derivation in \cite{Shapiro2009}, note that $\left( \Exp{|Z|^{2}} \right)^{1 / 2} = \|Z\|_2$, where $\|\cdot\|_2$ denotes the $L_2$ norm of the space $\mathcal L_2 (\Omega,\mathcal F,P_\theta)$. The norm may also be written as:
\begin{equation*}
    \|Z\|_2 = \sup_{\|\xi\|_2 \leq 1} \dotp{\xi}{Z},
\end{equation*}
and hence
\begin{equation*}
\begin{split}
  \left( \Exp{(Z - \Exp{Z})_{+}^{2}} \right)^{1 / 2} = \sup_{\|\xi\|_2 \leq 1} \dotp{\xi}{(Z - \Exp{Z})_{+}} &= \sup_{\|\xi\|_2 \leq 1, \xi \geq 0} \dotp{\xi}{Z - \Exp{Z}} \\
    &= \sup_{\|\xi\|_2 \leq 1, \xi \geq 0} \dotp{\xi - \Exp{\xi}}{Z}.
\end{split}
\end{equation*}
It follows that Eq. \eqref{eq:coherent_as_optimization} holds with
\begin{equation*}
    \U = \left\{ \xi' \in \cZ^* : \quad \xi' = 1 + c \xi - c \Exp{\xi}, \quad \|\xi\|_q \leq 1, \quad \xi \geq 0  \right\}.
\end{equation*}
For this case it will be more convenient to write Eq. \eqref{eq:coherent_as_optimization} in the following form
\begin{equation}\label{eq:msd_as_optimization}
    \msd(Z) = \sup_{\|\xi\|_q \leq 1, \xi \geq 0} \dotp{1 + c \xi - c \Exp{\xi}}{Z}.
\end{equation}
Let $\bar{\xi}$ denote an optimal solution for \eqref{eq:msd_as_optimization}. In \cite{Shapiro2009} it is shown that $\bar{\xi}$ is a contact point of $(Z - \Exp{Z})_{+}$, that is
\begin{equation*}
    \bar{\xi} \in \arg \max \left\{ \dotp{\xi}{(Z - \Exp{Z})_{+}}: \|\xi\|_2 \leq 1 \right\},
\end{equation*}
and we have that
\begin{equation}\label{eq:bar_xi_p2}
\bar{\xi} = \frac{(Z - \Exp{Z})_{+}}{\|(Z - \Exp{Z})_{+}\|_2} = \frac{(Z - \Exp{Z})_{+}}{\SD(Z)}.
\end{equation}
Note that $\bar{\xi}$ is not necessarily a probability distribution, but for $c\in[0,1]$, it can be shown \cite{Shapiro2009} that $1 + c \bar{\xi} - c \Exp{\bar{\xi}}$ always is.

In the following we show that $\bar{\xi}$ may be used to write the gradient $\dt{\msd(Z)}$ as an expectation, which will lead to a sampling algorithm for the gradient.

\begin{proposition}\label{prop:msd_grad_supp}
    Under Assumption \ref{ass:LR_well_behaved}, we have that
\begin{equation*}
    \dt{\msd(Z)} = \dt{\Exp{Z}} + \frac{c}{\SD(Z)}\Exp{(Z - \Exp{Z})_{+} \left(\dt \log P(\omega) (Z - \Exp{Z}) - \dt{\Exp{Z}}\right)},
\end{equation*}
and, according to the standard likelihood-ratio method,
\begin{equation*}
    \dt{\Exp{Z}} = \Exp{\dt \log P(\omega) Z}.
\end{equation*}
\end{proposition}
\begin{proof}
Note that in Eq. \eqref{eq:msd_as_optimization} the constraints do not depend on $\theta$. Therefore, using the envelope theorem we obtain that
\begin{equation}\label{eq:msd_proof_1}
\begin{split}
  \dt{\rho(Z)} &= \dt{ \dotp{1 + c\bar{\xi} - c\Exp{\bar{\xi}}}{Z} } \\
    &= \dt{ \dotp{1}{Z} } +c \dt{ \dotp{\bar{\xi}}{Z} } - c \dt{ \dotp{\Exp{\bar{\xi}}}{Z} }.
\end{split}
\end{equation}
We now write each of the terms in Eq. \eqref{eq:msd_proof_1} as an expectation.
We start with the following standard likelihood-ratio result:
\begin{equation*}
    \dt{\dotp{1}{Z}} = \dt{\Exp{Z}} = \Exp{\dt \log P(\omega) Z}.
\end{equation*}
Also, we have that
\begin{equation*}
    \dotp{\Exp{\bar{\xi}}}{Z} = \Exp{\bar{\xi}}\Exp{Z},
\end{equation*}
therefore, by the derivative of a product rule:
\begin{equation*}
    \dt{\dotp{\Exp{\bar{\xi}}}{Z}} = \dt{\Exp{\bar{\xi}}}\Exp{Z} + \Exp{\bar{\xi}}\dt{\Exp{Z}}.
\end{equation*}
By the likelihood-ratio trick and Eq. \eqref{eq:bar_xi_p2} we have that
\begin{equation*}
    \dt{\Exp{\bar{\xi}}} = \frac{1}{\SD(Z)} \Exp{\dt \log P(\omega) (Z - \Exp{Z})_{+} }.
\end{equation*}

Also, by the likelihood-ratio trick
\begin{equation*}
\dt{\Exp{\bar{\xi} Z}} = \Exp{\dt \log P(\omega) \bar{\xi} Z}.
\end{equation*}
Plugging these terms back in Eq. \eqref{eq:msd_proof_1}, we have that
\begin{equation*}
\begin{split}
  \dt{\rho(Z)} &= \dt{\Exp{Z}} + c\dt{\Exp{\bar{\xi} Z}} - c\dt{\Exp{\bar{\xi}}}\Exp{Z} -c\Exp{\bar{\xi}}\dt{\Exp{Z}} \\
    &= \dt{\Exp{Z}} + c\Exp{\bar{\xi} \left(\dt \log P(\omega) Z - \dt{\Exp{Z}}\right)} - c\dt{\Exp{\bar{\xi}}}\Exp{Z} \\
    &= \dt{\Exp{Z}} + \frac{c}{\SD(Z)}\Exp{(Z - \Exp{Z})_{+} \left(\dt \log P(\omega) Z - \dt{\Exp{Z}}\right)} - c\dt{\Exp{\bar{\xi}}}\Exp{Z} \\
    &= \dt{\Exp{Z}} + \frac{c}{\SD(Z)}\Exp{(Z - \Exp{Z})_{+} \left(\dt \log P(\omega) (Z - \Exp{Z}) - \dt{\Exp{Z}}\right)}.\\
\end{split}
\end{equation*}
\end{proof}

Proposition \ref{prop:msd_grad} naturally leads to a sampling-based gradient estimation algortihm, which we term \texttt{GMSD} (Gradient of Mean Semi-Deviation). The algorithm is described in Algorithm \ref{alg:Gmsd}.

\begin{algorithm}
\caption{\texttt{GMSD}}\label{alg:Gmsd}
1: \alggiven
\begin{itemize}
\item Risk level $c$
\item An i.i.d. sequence $z_{1},\dots,z_{N} \sim P_\theta$.
\end{itemize}
2: Set
\begin{equation*}
    \widehat{\Exp{Z}} = \avgN z_i.
\end{equation*}
3: Set
\begin{equation*}
    \widehat{\SD(Z)} = \left( \avgN ( z_i - \widehat{\Exp{Z}} )_{+}^{2} \right)^{1 / 2}.
\end{equation*}
4: Set
\begin{equation*}
    \widehat{\dt{\Exp{Z}}} = \avgN \dt \log P(z_i) z_i.
\end{equation*}
5: \algreturn
\begin{equation*}
    \hat{\dt{\rho(Z)}} = \widehat{\dt{\Exp{Z}}} + \frac{c}{\widehat{\SD(Z)}}\avgN (z_i - \widehat{\Exp{Z}})_{+} \left(\dt \log P(z_i) (z_i - \widehat{\Exp{Z}}) - \widehat{\dt{\Exp{Z}}}\right).
\end{equation*}
\end{algorithm}

%%%%%%%%%%%%%%%%%%%%%%%%%%%%%%%%%%%%%%%%%%%%%%%%%%%%%%%%%%%%%%%%%%%%%%%%%%%%%%%%
%%%%%%%%%%%%%%%%%%%%%%%%%%%%%%%%%%%%%%%%%%%%%%%%%%%%%%%%%%%%%%%%%%%%%%%%%%%%%%%%
%%%%%%%%%%%%%%%%%%%%%%%%%%%%%%%%%%%%%%%%%%%%%%%%%%%%%%%%%%%%%%%%%%%%%%%%%%%%%%%%
%%%%%%%%%%%%%%%%%%%%%%%%%%%%%%%%%%%%%%%%%%%%%%%%%%%%%%%%%%%%%%%%%%%%%%%%%%%%%%%%
%%%%%%%%%%%%%%%%%%%%%%%%%%%%%%%%%%%%%%%%%%%%%%%%%%%%%%%%%%%%%%%%%%%%%%%%%%%%%%%%

\section{Consistency Proof}

Let $\left( \Omega_{SAA},\mathcal F_{SAA},P_{SAA} \right)$ denote the probability space of the SAA functions (i.e., the randomness due to sampling).

Let $L_{\theta;N}(\xi,\lambda^{\mathcal P},\lambda^{\mathcal E},\lambda^{\mathcal I})$ denote the Lagrangian of the SAA problem
\begin{equation}\label{eq:SAA_Lagrangian}
\begin{split}
L_{\theta;N}(\xi,\lambda^{\mathcal P},\lambda^{\mathcal E},\lambda^{\mathcal I})=&\sum_{\omega \in \Omega} \!\! \xi(\omega) P_{\theta;N}(\omega) Z(\omega)\!-\!\lambda^{\mathcal P}\left(\sum_{\omega \in \Omega}\xi(\omega)P_{\theta;N}(\omega)\!-\!1\!\right)\\
&-\sum_{e\in\mathcal E}\lambda^{\mathcal E}(e) f_e(\xi,P_{\theta;N})-\sum_{i\in\mathcal I}\lambda^{\mathcal I}(i) f_i(\xi,P_{\theta;N}).
\end{split}
\end{equation}
Recall that $\spset \subset \reals^{|\Omega|}\times \reals\times \reals^{|\mathcal E|}\times \reals^{|\mathcal I|}_{+}$ denotes the set of saddle points of the true Lagrangian \eqref{eq:Lagrangian}.
Let $\saaspset\subset \reals^{|\Omega|}\times \reals\times \reals^{|\mathcal E|}\times \reals^{|\mathcal I|}_{+}$ denote the set of SAA Lagrangian \eqref{eq:SAA_Lagrangian} saddle points.

Suppose that there exists a compact set $C \equiv C_\xi \times C_\lambda$, where $C_\xi \subset \reals^{|\Omega|}$ and $C_\lambda \subset \reals\times \reals^{|\mathcal E|}\times \reals^{|\mathcal I|}_{+}$ such that:
\begin{description}
\item[(i)] The set of Lagrangian saddle points $\spset \subset C$ is non-empty and bounded.
\item[(ii)] The functions $f_e(\xi,P_\theta)$ for all $e\in\mathcal E$ and $f_i(\xi,P_\theta)$ for all $i\in\mathcal I$ are finite valued and continuous (in $\xi$) on $C_\xi$.
\item[(iii)] For $N$ large enough the set $\saaspset$ is non-empty and $\saaspset \subset C$ w.p.~1.
\end{description}

Recall from Assumption \ref{assume:risk_envelope} that for each fixed $\xi\in\mathcal B$, both $f_i(\xi,p)$ and $g_e(\xi,p)$ are continuous in $p$. Furthermore, by the S.L.L.N. of Markov chains, for each policy parameter, we have $P_{\theta,N}\rightarrow P_{\theta}$ w.p.~1. From the definition of the Lagrangian function and continuity of constraint functions,  one can easily see that for each $(\xi,\lambda^{\mathcal P},\lambda^{\mathcal E},\lambda^{\mathcal I})\in\reals^{|\Omega|}\times\reals\times\reals^{|\mathcal E|}\times\reals_{+}^{|\mathcal I|}$, $L_{\theta;N}(\xi,\lambda^{\mathcal P},\lambda^{\mathcal E},\lambda^{\mathcal I})\rightarrow L_{\theta}(\xi,\lambda^{\mathcal P},\lambda^{\mathcal E},\lambda^{\mathcal I})$ w.p.~1. Denote with $\setdist{A}{B}$ the deviation of set $A$ from set $B$, i.e., $\setdist{A}{B}=\sup_{x\in A}\inf_{y\in B}\|x-y\|$.  Further assume that:
\begin{description}
\item[(iv)] If $\xi_N \in \mathcal U(\pemp)$ and $\xi_N$ converges w.p.~1 to a point $\xi$, then $\xi\in\mathcal U(\pprob)$.
\end{description}
According to the discussion in Page 161 of \citet{Shapiro2009}, the Slater condition of Assumption \ref{assume:risk_envelope} guarantees the following condition:
\begin{description}
\item[(v)] For some point $\xi \in \optset$ there exists a sequence $\xi_N \in \mathcal U(\pemp)$ such that $\xi_N \to \xi$ w.p.~1,
\end{description}
and from Theorem 6.6 in \citet{Shapiro2009}, we know that both sets $ \mathcal U(\pemp)$ and $ \mathcal U(P_\theta)$ are convex and compact.
Furthermore, note that we have
\begin{description}
\item[(vi)] The objective function on \eqref{eq:coherent_as_optimization} is linear, finite valued and continuous in $\xi$ on $C_\xi$ (these conditions obviously hold for almost all $\omega\in\Omega$ in the integrand function $\xi(\omega)Z(\omega)$).
\item[(vii)] S.L.L.N. holds point-wise for any $\xi$.
\end{description}
From (i,iv,v,vi,vii), and under the same lines of proof as in Theorem 5.5 of \citet{Shapiro2009}, we have that
\begin{equation}\label{eq:consistency_prf_1}
    \rho_N(Z) \to \rho(Z) \textrm{ w.p. } 1 \textrm{ as } N \to \infty,
\end{equation}
\begin{equation}\label{eq:consistency_prf_2}
    \setdist{\saaoptset}{\optset} \to 0 \textrm{ w.p. } 1 \textrm{ as } N \to \infty,
\end{equation}

In part 1 and part 2 of the following proof, we show, by following similar derivations as in Theorem 5.2, Theorem 5.3 and Theorem 5.4 of \citet{Shapiro2009}, that $L_{\theta;N}(\xi^*_{\theta;N},\lambda^{*,\mathcal P}_{\theta;N},\lambda^{*,\mathcal E}_{\theta;N},\lambda^{*,\mathcal I}_{\theta;N})\rightarrow L_{\theta}(\xi^*_{\theta},\lambda^{*,\mathcal P}_{\theta},\lambda^{*,\mathcal E}_{\theta},\lambda^{*,\mathcal I}_{\theta})$ w.p.~1 and $\setdist{\saaspset}{\spset} \to 0 \textrm{ w.p. } 1 \textrm{ as } N \to \infty$. Based on the definition of the deviation of sets, the limit point of any element in $\saaspset$ is also an element in $\spset$.

Assumptions (i) and (iii) imply that we can restrict our attention to the set $C$.

\paragraph{Part 1}
We first show that $L_{\theta;N}(\xi^*_{\theta;N},\lambda^{*,\mathcal P}_{\theta;N},\lambda^{*,\mathcal E}_{\theta;N},\lambda^{*,\mathcal I}_{\theta;N})$ converges to $L_{\theta}(\xi^*_{\theta},\lambda^{*,\mathcal P}_{\theta},\lambda^{*,\mathcal E}_{\theta},\lambda^{*,\mathcal I}_{\theta})$ w.p.~1 as $N\to \infty$.

%%%%%%%%%%%%%%%%%%%%%%%%%
For each fixed $(\lambda^{\mathcal P},\lambda^{\mathcal E},\lambda^{\mathcal I})\in C_\lambda$, the function $L_{\theta}(\xi,\lambda^{\mathcal P},\lambda^{\mathcal E},\lambda^{\mathcal I})$ is convex and continuous in $\xi$.  Together with the point-wise S.L.L.N. property, Theorem 7.49 of \citet{Shapiro2009} implies that $L_{\theta;N}(\xi,\lambda^{\mathcal P},\lambda^{\mathcal E},\lambda^{\mathcal I}) - L_{\theta}(\xi,\lambda^{\mathcal P},\lambda^{\mathcal E},\lambda^{\mathcal I})\overset{e}{\rightarrow}0$, where $\overset{e}{\rightarrow}$ denotes epi-convergence.  Furthermore, since the objective and constraint functions are convex in $\xi$ and are finite valued on $C_\xi$, the set $\text{dom} L_{\theta}(\cdot,\lambda^{\mathcal P},\lambda^{\mathcal E},\lambda^{\mathcal I})$ has non-empty interior. It follows from Theorem 7.27 of  \citet{Shapiro2009} that epi-convergence of $L_{\theta,N}$ to $L_\theta$ implies uniform convergence on $C_\xi$, i.e., $ \sup_{\xi \in C_\xi} \left| L_{\theta;N}(\xi,\lambda^{\mathcal P},\lambda^{\mathcal E},\lambda^{\mathcal I}) - L_{\theta}(\xi,\lambda^{\mathcal P},\lambda^{\mathcal E},\lambda^{\mathcal I})\right| \leq \epsilon
$.
On the other hand, for each fixed $\xi\in C_\xi$,
the function $L_{\theta}(\xi,\lambda^{\mathcal P},\lambda^{\mathcal E},\lambda^{\mathcal I})$ is linear and thus continuous in $(\lambda^{\mathcal P},\lambda^{\mathcal E},\lambda^{\mathcal I})$ and $\text{dom} L_{\theta}(\xi,\cdot,\cdot,\cdot)=\reals\times\reals^{|\mathcal E|}\times\reals^{|\mathcal I|}$ has non-empty interior. It follows from analogous arguments that $ \sup_{(\lambda^{\mathcal P},\lambda^{\mathcal E},\lambda^{\mathcal I}) \in C_{\lambda}} \left| L_{\theta;N}(\xi,\lambda^{\mathcal P},\lambda^{\mathcal E},\lambda^{\mathcal I}) - L_{\theta}(\xi,\lambda^{\mathcal P},\lambda^{\mathcal E},\lambda^{\mathcal I})\right| \leq \epsilon$. Combining these results implies
 that for any $\epsilon > 0$ and a.e. $\omega_{SAA} \in \Omega_{SAA}$ there is a $N^*(\epsilon,\omega_{SAA})$ such that
\begin{equation}\label{eq:consistency_prf_3}
    \sup_{(\xi,\lambda^{\mathcal P},\lambda^{\mathcal E},\lambda^{\mathcal I}) \in C} \left| L_{\theta;N}(\xi,\lambda^{\mathcal P},\lambda^{\mathcal E},\lambda^{\mathcal I}) - L_{\theta}(\xi,\lambda^{\mathcal P},\lambda^{\mathcal E},\lambda^{\mathcal I})\right| \leq \epsilon.
\end{equation}
Now, assume by contradiction that for some $N > N^*(\epsilon,\omega_{SAA})$ we have $\saaLag - \Lag > \epsilon$. Then by definition of the saddle points
\begin{equation*}
\begin{split}
    L_{\theta;N}(\xi^*_{\theta;N},\lambda^{*,\mathcal P}_{\theta},\lambda^{*,\mathcal E}_{\theta},\lambda^{*,\mathcal I}_{\theta}) &\geq \saaLag  \\ &> \Lag + \epsilon \geq L_{\theta}(\xi^*_{\theta;N},\lambda^{*,\mathcal P}_{\theta},\lambda^{*,\mathcal E}_{\theta},\lambda^{*,\mathcal I}_{\theta}) + \epsilon,
\end{split}
\end{equation*}
contradicting \eqref{eq:consistency_prf_3}.

Similarly, assuming by contradiction that $ \Lag - \saaLag> \epsilon$ gives
\begin{equation*}
\begin{split}
    L_{\theta}(\xi^*_{\theta},\lambda^{*,\mathcal P}_{\theta;N},\lambda^{*,\mathcal E}_{\theta;N},\lambda^{*,\mathcal I}_{\theta;N}) &\geq \Lag \\ &> \saaLag  + \epsilon \geq L_{\theta;N}(\xi^*_{\theta},\lambda^{*,\mathcal P}_{\theta;N},\lambda^{*,\mathcal E}_{\theta;N},\lambda^{*,\mathcal I}_{\theta;N}) + \epsilon,
\end{split}
\end{equation*}
also contradicting \eqref{eq:consistency_prf_3}.

It follows that $\left| \saaLag - \Lag \right| \leq \epsilon$ for all $N > N^*(\epsilon,\omega_{SAA})$, and therefore
\begin{equation}\label{eq:consistency_prf_L_converges}
  \lim_{N\to \infty} L_{\theta;N}(\xi^*_{\theta;N},\lambda^{*,\mathcal P}_{\theta;N},\lambda^{*,\mathcal E}_{\theta;N},\lambda^{*,\mathcal I}_{\theta;N}) = L_{\theta}(\xi^*_{\theta},\lambda^{*,\mathcal P}_{\theta},\lambda^{*,\mathcal E}_{\theta},\lambda^{*,\mathcal I}_{\theta}),
\end{equation}
w.p.~1.

\paragraph{Part 2}
Let us now show that $\setdist{\saaspset}{\spset} \to 0$. We argue by a contradiction. Suppose that $\setdist{\saaspset}{\spset} \nrightarrow 0$. Since $C$ is compact, we can assume that there exists a sequence $(\xi^*_{\theta;N},\lambda^{*,\mathcal P}_{\theta;N},\lambda^{*,\mathcal E}_{\theta;N},\lambda^{*,\mathcal I}_{\theta;N}) \in \saaspset$ that converges to a point $(\bar{\xi}^*,\bar{\lambda}^{*,\mathcal P},\bar{\lambda}^{*,\mathcal E},\bar{\lambda}^{*,\mathcal I})\in C$ and
$(\bar{\xi}^*,\bar{\lambda}^{*,\mathcal P},\bar{\lambda}^{*,\mathcal E},\bar{\lambda}^{*,\mathcal I}) \not\in \spset$. However, from \eqref{eq:consistency_prf_2} we must have that $\bar{\xi}^* \in \optset$. Therefore, we must have that
\begin{equation*}
    L_{\theta}(\bar{\xi}^*,\bar{\lambda}^{*,\mathcal P},\bar{\lambda}^{*,\mathcal E},\bar{\lambda}^{*,\mathcal I}) > L_{\theta}(\bar{\xi}^*,\lambda^{*,\mathcal P}_{\theta},\lambda^{*,\mathcal E}_{\theta},\lambda^{*,\mathcal I}_{\theta}),
\end{equation*}
by definition of the saddle point set.

Now,
\begin{equation}\label{eq:consistency_prf_4}
\begin{split}
    &L_{\theta;N}(\xi^*_{\theta;N},\lambda^{*,\mathcal P}_{\theta;N},\lambda^{*,\mathcal E}_{\theta;N},\lambda^{*,\mathcal I}_{\theta;N}) - L_{\theta}(\bar{\xi}^*,\bar{\lambda}^{*,\mathcal P},\bar{\lambda}^{*,\mathcal E},\bar{\lambda}^{*,\mathcal I}) \\=& \left[L_{\theta;N}(\xi^*_{\theta;N},\lambda^{*,\mathcal P}_{\theta;N},\lambda^{*,\mathcal E}_{\theta;N},\lambda^{*,\mathcal I}_{\theta;N}) - L_{\theta}(\xi^*_{\theta;N},\lambda^{*,\mathcal P}_{\theta;N},\lambda^{*,\mathcal E}_{\theta;N},\lambda^{*,\mathcal I}_{\theta;N})\right] + \\ &+ \left[L_{\theta}(\xi^*_{\theta;N},\lambda^{*,\mathcal P}_{\theta;N},\lambda^{*,\mathcal E}_{\theta;N},\lambda^{*,\mathcal I}_{\theta;N}) - L_{\theta}(\bar{\xi}^*,\bar{\lambda}^{*,\mathcal P},\bar{\lambda}^{*,\mathcal E},\bar{\lambda}^{*,\mathcal I})\right].
\end{split}
\end{equation}
The first term in the r.h.s. of \eqref{eq:consistency_prf_4} tends to zero, using the argument from  \eqref{eq:consistency_prf_3}, and the second by continuity of $L_{\theta}$ guaranteed by (ii). We thus obtain that $L_{\theta;N}(\xi^*_{\theta;N},\lambda^{*,\mathcal P}_{\theta;N},\lambda^{*,\mathcal E}_{\theta;N},\lambda^{*,\mathcal I}_{\theta;N})$ tends to $L_{\theta}(\bar{\xi}^*,\bar{\lambda}^{*,\mathcal P},\bar{\lambda}^{*,\mathcal E},\bar{\lambda}^{*,\mathcal I}) > L_{\theta}(\xi^*_{\theta},\lambda^{*,\mathcal P}_{\theta},\lambda^{*,\mathcal E}_{\theta},\lambda^{*,\mathcal I}_{\theta})$, which is a contradiction to \eqref{eq:consistency_prf_L_converges}.

\paragraph{Part 3}
We now show the consistency of $\dtN \rho(Z)$.

Consider Eq. \eqref{eq:SAA_gradient}. Since $\dt \log P(\cdot)$ is bounded by Assumption \ref{ass:LR_well_behaved}, and $\dt{f_i(\cdot;P_{\theta})}$ and $\dt{g_e(\cdot;P_{\theta})}$ are bounded by Assumption \ref{assume:risk_envelope}, and using our previous result  $\setdist{\saaspset}{\spset} \to 0$, we have that for a.e. $\omega_{SAA}\in \Omega_{SAA}$

\begin{equation*}
\begin{split}
  \lim_{N\to\infty}\dtN{\rho(Z)}  &=  \sum_{\omega\in\Omega} \pprob(\omega) \xi^*_{\theta}(\omega) \dt \log P(\omega) (Z(\omega) - \lambda^{*,\mathcal P}_{\theta}) \\
    & - \sum_{e\in\mathcal E} \lambda^{*,\mathcal E}_{\theta}(e) \dt{g_e(\xi^*_{\theta};P_{\theta})}\\
    & -\sum_{i\in\mathcal I} \lambda^{*,\mathcal I}_{\theta}(i) \dt{f_i(\xi^*_{\theta};P_{\theta})} \\
    &= \dt{\rho(Z)}.
\end{split}
\end{equation*}
where the first equality is obtained from the Envelop theorem (see Theorem \ref{thm:static_gradient}) with $(\xi^*_{\theta},\lambda^{*,\mathcal P}_{\theta},\lambda^{*,\mathcal E}_{\theta},\lambda^{*,\mathcal I}_{\theta})\in \saaspset\cap\spset$ is the limit point of the converging sequence $\{(\xi^*_{\theta;N},\lambda^{*,\mathcal P}_{\theta;N},\lambda^{*,\mathcal E}_{\theta;N},\lambda^{*,\mathcal I}_{\theta;N})\}_{N\in\mathbb N}$.

%%%%%%%%%%%%%%%%%%%%%%%%%%%%%%%%%%%%%%%%%%%%%%%%%%%%%%%%%%%%%%%%%%%%%%%%%%%%%%%%
%%%%%%%%%%%%%%%%%%%%%%%%%%%%%%%%%%%%%%%%%%%%%%%%%%%%%%%%%%%%%%%%%%%%%%%%%%%%%%%%
%%%%%%%%%%%%%%%%%%%%%%%%%%%%%%%%%%%%%%%%%%%%%%%%%%%%%%%%%%%%%%%%%%%%%%%%%%%%%%%%
%%%%%%%%%%%%%%%%%%%%%%%%%%%%%%%%%%%%%%%%%%%%%%%%%%%%%%%%%%%%%%%%%%%%%%%%%%%%%%%%
%%%%%%%%%%%%%%%%%%%%%%%%%%%%%%%%%%%%%%%%%%%%%%%%%%%%%%%%%%%%%%%%%%%%%%%%%%%%%%%%
\section{Proof of Theorem \ref{thm:dynamic_risk}}

Similar to the proof of Theorem \ref{thm:static_gradient}, recall the saddle point definition of $(\xi^*_{\theta,x},\lambda^{*,\mathcal P}_{\theta,x},\lambda^{*,\mathcal E}_{\theta,x},\lambda^{*,\mathcal I}_{\theta,x}) \in \spset$ and strong duality result, i.e.,
\begin{equation*}
\begin{split}
\max_{\xi\,:\,\xi P_\theta(\cdot|x)\in\mathcal U(x,P_\theta(\cdot|x))}\sum_{x'  \in \mathcal X} \, \xi(x') P_\theta(x'|x)V_\theta(x')&=\max_{\xi\geq 0}\min_{\lambda^{\mathcal P},\lambda^{\mathcal I}\geq 0,\lambda^{\mathcal E}}L_{\theta,x}(\xi,\lambda^{\mathcal P},\lambda^{\mathcal E},\lambda^{\mathcal I})\\&=\min_{\lambda^{\mathcal P},\lambda^{\mathcal I}\geq 0,\lambda^{\mathcal E}}\max_{\xi\geq 0}L_{\theta,x}(\xi,\lambda^{\mathcal P},\lambda^{\mathcal E},\lambda^{\mathcal I}).
\end{split}
\end{equation*}
the gradient formula in \eqref{eq:envelope_theorem} can be written as
\[
\begin{split}
\nabla_\theta V_\theta(x)&=\nabla_\theta \left[C_\theta(x) \!+\! \gamma\max_{\xi\,:\,\xi P_\theta(\cdot |x)\in \U(x,P_\theta(\cdot |x))}\mathbb E_{\xi}[V_\theta]\right]\\
&=\gamma\! \sum_{x'  \in \mathcal X} \! \xi^*_{\theta,x}(x') P_\theta(x'|x)\nabla_\theta V_\theta(x')+\sum_{a\in\mathcal A}\mu_\theta(a|x)\nabla_\theta\log\mu_\theta(a|x)h_\theta(x,a),
\end{split}
\]
where the stage-wise cost function $h_\theta(x,a)$ is defined in \eqref{eq:h}. By defining $\widehat{h}_\theta(x)=\sum_{a\in\mathcal A}\mu_\theta(a|x)\nabla_\theta\log\mu_\theta(a|x)h_\theta(x,a)$ and unfolding the recursion, the above expression implies
\[
\begin{split}
\nabla_\theta V_\theta(x_0)=&\widehat{h}_\theta(x_0)+\gamma\sum_{x_1\in\mathcal X}P_\theta(x_1|x_0)\xi_\theta^*(x_1)\Bigg[\widehat{h}_\theta(x_1)+\gamma \sum_{x_2\in\mathcal X}P_\theta(x_2|x_1)\xi_\theta^*(x_2) \nabla_\theta V_\theta\left(x_2\right)\Bigg].
\end{split}
\]
Now since $\nabla_\theta V_\theta$ is continuously differentiable with bounded derivatives, when $t\rightarrow\infty$, one obtains $\gamma^t\nabla_\theta V_\theta(x)\rightarrow 0$ for any $x\in\mathcal X$. Therefore, by Bounded Convergence Theorem, $\lim_{t\rightarrow\infty}\rho(\gamma^t V_{\theta}(x_t))=0$, when $x_0=x$ the above expression implies the result of this theorem.

%%%%%%%%%%%%%%%%%%%%%%%%%%%%%%%%%%%%%%%%%%%%%%%%%%%%%%%%%%%%%%%%%%%%%%%%%%%%%%%%
%%%%%%%%%%%%%%%%%%%%%%%%%%%%%%%%%%%%%%%%%%%%%%%%%%%%%%%%%%%%%%%%%%%%%%%%%%%%%%%%
%%%%%%%%%%%%%%%%%%%%%%%%%%%%%%%%%%%%%%%%%%%%%%%%%%%%%%%%%%%%%%%%%%%%%%%%%%%%%%%%
%%%%%%%%%%%%%%%%%%%%%%%%%%%%%%%%%%%%%%%%%%%%%%%%%%%%%%%%%%%%%%%%%%%%%%%%%%%%%%%%
%%%%%%%%%%%%%%%%%%%%%%%%%%%%%%%%%%%%%%%%%%%%%%%%%%%%%%%%%%%%%%%%%%%%%%%%%%%%%%%%

\section{Gradient Formula for Dynamic Risk - Full Results}
\label{sec:dynamic}

In this section, we first derive a new formula for the gradient of a general Markov-coherent dynamic risk measure $\dt \rho_\infty(\mdp)$ that involves the \emph{value function} of the risk objective $\rho_\infty(\mdp)$ (e.g.,~the value function proposed by~\citealt{ruszczynski2010risk}). This formula extends the well-known ``policy gradient theorem"~\cite{sutton_policy_2000,konda2000actor} developed for the expected return to Markov-coherent dynamic risk measures. Using this formula, we suggest the following actor-critic style algorithm for estimating $\dt \rho_\infty(\mdp)$:

%\begin{itemize}
%\item
\textbf{Critic:} For a given policy $\theta$, calculate the \emph{risk-sensitive value function} of $\rho_\infty(\mdp)$ (see Section~\ref{sec:val_fn_rpprox}), and \\
%\item
\textbf{Actor:} Using the critic's value function, estimate $\dt \rho_\infty(\mdp)$ by sampling (see Section~\ref{sec:pol_grad}).
%\end{itemize}

The value function proposed by~\citet{ruszczynski2010risk} assigns to each state a particular value that encodes the long-term risk starting from that state. When the state space $\mathcal X$ is large, calculating the value function by dynamic programming (as suggested by~\citealt{ruszczynski2010risk}) becomes intractable due to the ``curse of dimensionality". For the risk-neutral case, a standard solution to this problem is to approximate the value function by a set of state-dependent features, and use sampling to calculate the parameters of this approximation~\citep{BT96}. In particular, \emph{temporal difference} (TD) learning methods~\cite{sutton_reinforcement_1998} are popular for this purpose, which have been recently extended to robust MDPs by~\citet{tamar2014robust}. We use their (robust) TD algorithm and show how our critic use it to approximates the {\em risk-sensitive} value function. We then discuss how the error introduced by this approximation affects the gradient estimate of the actor.
%by projected \emph{risk-sensitive Bellman iteration}
% Explain general actor-critic approach - value function estimation + policy gradient theorem

%%%%%%%%%%%%%%%%%%%%%%%%%%%%%%%%%%%%%%%%%%%%%%%%%%%%%%%%%%%%%%%%%%%%%%%%%%%%%%%%
%%%%%%%%%%%%%%%%%%%%%%%%%%%%%%%%%%%%%%%%%%%%%%%%%%%%%%%%%%%%%%%%%%%%%%%%%%%%%%%%
%%%%%%%%%%%%%%%%%%%%%%%%%%%%%%%%%%%%%%%%%%%%%%%%%%%%%%%%%%%%%%%%%%%%%%%%%%%%%%%%
\subsection{Dynamic Risk}
We provide a multi-period generalization of the concepts presented in Section~\ref{subsec:coherent}. Here we closely follow the discussion in~\citet{ruszczynski2010risk}.

Consider a probability space $(\Omega, \mathcal F, P_\theta)$, a filtration $\mathcal F_0\subset \mathcal F_1\subset \mathcal F_2 \cdots \subset \mathcal F_T \subset \mathcal F$, and an adapted sequence of real-valued random variables $Z_t$, $t\in \{0, \ldots,T\}$. We assume that $\mathcal F_0 = \{\Omega, \emptyset\}$, i.e.,~$Z_0$ is deterministic. For each $t\in\{0, \ldots, T\}$, we denote by $\mathcal Z_t$ the space of random variables defined over the probability space $(\Omega, \mathcal F_t, P_\theta)$, and also let $\mathcal Z_{t, T}:=\mathcal Z_t \times \cdots \times \mathcal Z_T$ be a sequence of these spaces. The sequence of random variables $Z_t$ can be interpreted as the stage-wise costs observed along a trajectory generated by an MDP parameterized by a parameter $\theta$, i.e.,~$Z_{0,T} \doteq \big(Z_0=\gamma^0C(x_0,a_0),\dots,Z_T=\gamma^TC(x_T,a_T)\big)\in\mathcal Z_{0,T}$.

% Present the trajectory as the fundamental random variable

In particular, we are interested in the sequence of random variables induced by the trajectories from a Markov decision process (MDP) parameterized by parameter $\theta$.

Explicitly, for any $t\geq 0$ and state dependent random variable $Z(x_{t+1})\in\mathcal Z_{t+1}$, the risk evaluation is given by

%\vspace{-0.2in}
\begin{small}
\begin{equation}
\label{eq:representation-result}
\rho\big(Z(x_{t+1})\big)=\max_{\xi\,:\, \xi P_\theta(\cdot |x_t)\in \U(x_t,P_\theta(\cdot |x_t))}\mathbb E_{\xi}\big[Z(x_{t+1})\big],
\end{equation}
\end{small}
%\vspace{-0.2in}

where we let $\U(x_t,P_\theta(\cdot |x_t))$ denote the risk-envelope \eqref{eq:U_as_optimization} with $\pprob$ replaced with $P_\theta(\cdot |x_t)$. The Markovian assumption on the risk measure $\rho_T(\mdp)$ allows us to optimize it using dynamic programming techniques.
\subsection{Risk-Sensitive Bellman Equation}
\label{subsec:Risk-Bellman}

Our value-function estimation method is driven by a Bellman-style equation for Markov coherent risks. Let $B(\mathcal X)$ denote the space of real-valued bounded functions on $\mathcal X$ and $C_\theta(x)=\sum_{a\in\mathcal A}C(x,a)\mu_\theta(a|x)$ be the stage-wise cost function induced by policy $\mu_\theta$. We now define the risk sensitive Bellman operator $T_{\theta}[V] : B(\mathcal X) \mapsto B(\mathcal X)$ as
\begin{equation}\label{eq:T_supp}
T_{\theta}[V](x) :=C_\theta(x) + \gamma\max_{\xi P_\theta(\cdot |x)\in \U(x,P_\theta(\cdot |x))}\mathbb E_{\xi}[V].
\end{equation}
According to Theorem~1 in~\citet{ruszczynski2010risk}, the operator $T_\theta$ has a unique fixed-point $V_\theta$, i.e.,~ $T_{\theta}[V_\theta](x)=V_\theta(x),\;\forall x\in\mathcal{X}$, that is equal to the risk objective function induced by $\theta$, i.e.,~$V_\theta(x_0)=\rho_\infty(\mdp)$. However, when the state space $\mathcal X$ is large, exact enumeration of the Bellman equation is intractable due to ``curse of dimensionality''. Next, we provide an iterative approach to approximate the risk sensitive value function.

%%%%%%%%%%%%%%%%%%%%%%%%%%%%%%%%%%%%%%%%%%%%%%%%%%%%%%%%%%%%%%%%%%%%%%%%%%%%%%%%
%%%%%%%%%%%%%%%%%%%%%%%%%%%%%%%%%%%%%%%%%%%%%%%%%%%%%%%%%%%%%%%%%%%%%%%%%%%%%%%%
%%%%%%%%%%%%%%%%%%%%%%%%%%%%%%%%%%%%%%%%%%%%%%%%%%%%%%%%%%%%%%%%%%%%%%%%%%%%%%%%

\subsection{Value Function Approximation}
\label{sec:val_fn_rpprox}

Consider the linear approximation of the risk-sensitive value function $V_\theta(x)\approx v^\top\phi(x)$, where $\phi(\cdot)\in\mathbb R^{\kappa_2}$ is the $\kappa_2$-dimensional state-dependent feature vector. Thus, the approximate value function belongs to the low dimensional sub-space $\mathcal{V}=\left\{\Phi v|v\in\mathbb R^{\kappa_2}\right\}$, where $\Phi:\mathcal X\rightarrow\mathbb R^{\kappa_2}$ is a function mapping such that $\Phi(x)=\phi(x)$. %Let $v_\theta^*\in\mathbb R^{\kappa_2}$ be the best approximation parameter, and thus, $\tilde V_{\theta}(x)=\phi(x)^\top v_\theta^*$ be the best linear approximation of $ V_{\theta}(x)$.
The goal of our critric is to find a good approximation of $V_\theta$ from simulated trajectories of the MDP. In order to have a well-defined approximation scheme, we first impose the following standard assumption~\citep{BT96}.
\begin{assumption}
The mapping $\Phi$ has full column rank.
\end{assumption}
For a function $y:\mathcal X\rightarrow\mathbb R$, we define its weighted (by $d$) $\ell_2$-norm as $\|y\|_d=\sqrt{\sum_{x'}d(x^\prime|x)y(x^\prime)^2}$, where $d$ is a distribution over $\mathcal X$. Using this, we define $\Pi:\mathcal X\rightarrow\mathcal{V}$, the orthogonal projection from $\mathbb R$ to $\mathcal V$, w.r.t.~a norm weighted by the stationary distribution of the policy, $d_\theta(x'|x)$.

Note that the TD methods approximate the value function $V_\theta$ with the fixed-point of the joint operator $\Pi T_\theta$, i.e.,~$\tilde V_\theta(x)=v_\theta^{*\top}\phi(x)$, such that
\begin{equation}\label{eq:projected_fixed_point}
\forall x\in\mathcal X, \quad\quad\quad \tilde V_\theta(x)=\Pi T_{\theta} [\tilde V_\theta](x).
\end{equation}
From Eq.~\ref{eq:representation-result} that has been derived from Theorem~\ref{thm:rep} for dynamic risks, it is easy to see that the risk-sensitive Bellman equation~\eqref{eq:T_supp} is a robust Bellman equation~\citep{nilim_robust_2005} with uncertainty set $\mathcal U(x,P_\theta(\cdot|x))$. Thus, we may use the TD approximation of the robust Bellman equation proposed by~\citet{tamar2014robust} to find an approximation of $V_\theta$.
%In this section we mainly summarize the value function approximation results of robust projected Bellman iteration from \citet{tamar2014robust}.
%
We will need the following assumption analogous to Assumption~2 in~\citet{tamar2014robust}.
\begin{assumption}\label{assume:risk}
There exists $\kappa \in (0, 1)$ such that $\xi(x') \leq \kappa/\gamma$, for all $\xi(\cdot) P_\theta(\cdot|x) \in \mathcal U(x, P_\theta(\cdot|x))$ and all $x,x' \in\mathcal X$.
\end{assumption}
Given Assumption~\ref{assume:risk}, Proposition~3 in~\citet{tamar2014robust} guarantees that the projected risk-sensitive Bellman operator $\Pi T_{\theta}$ is a contraction w.r.t.~$d_{\theta}$-norm. Therefore, Eq.~\ref{eq:projected_fixed_point} has a unique fixed-point solution $\tilde V_\theta(x)=v_\theta^{*\top}\phi(x)$. This means that $v_\theta^*\in\mathbb R^{\kappa_2}$ satisfies $v_\theta^*\in\arg\min_{v}\|T_{\theta}[\Phi v]-\Phi v\|_{d_\theta}^2$. By the projection theorem on Hilbert spaces, the orthogonality condition for $v_\theta^*$ becomes
\begin{align*}
%\begin{split}
\sum_{x\in\mathcal X}&d_\theta(x|x_0)\phi(x)\phi(x)^\top v_\theta^*= \sum_{x\in\mathcal X} d_\theta(x|x_0)\phi(x) C_\theta(x)\\
&+\gamma \sum_{x\in\mathcal X}d_\theta(x|x_0)\phi(x)\max_{\xi\,:\,\xi P_\theta(\cdot |x) \in \U(x,P_\theta(\cdot |x))} \mathbb E_{\xi}[\Phi v_\theta^* ].
%\end{split}
\end{align*}
As a result, given a long enough trajectory $x_0$, $a_0$, $x_1$, $a_1$, $\ldots$, $x_{N-1}$, $a_{N-1}$ generated by policy $\theta$, we may estimate the fixed-point solution $v^*_\theta$ using the projected risk sensitive value iteration (PRSVI) algorithm with the update rule
\begin{align}
v_{k+1}&=\left(\frac{1}{N}\sum_{t=0}^{N-1} \phi(x_t)\phi(x_t)^{\top}\right)^{-1}\bigg[\frac{1}{N}\sum_{t=0}^{N-1}\phi(x_t) C_{\theta}(x_t)\nonumber\\
&+\gamma\frac{1}{N}\sum_{t=0}^{N-1} \phi(x_t)\max_{\xi P_\theta(\cdot |x_t)\in \U(x_t,P_\theta(\cdot |x_t))} \mathbb E_{\xi}[\Phi v_k ]\bigg].
\label{eq:PRSVI}
\end{align}
Note that using the law of large numbers, as both $N$ and $k$ tend to infinity, $v_k$ converges w.p.~1 to $v_\theta^*$, the unique solution of the fixed point equation $\Pi T_{\theta} [\Phi v]=\Phi v$.

%\subsection{SAA Formulation of $\max_{\xi P_\theta(\cdot |x)\in \U(x,P_\theta(\cdot |x))} \mathbb E_{\xi}[\Phi v]$}\label{sec:val_fn_rpprox_SAA}

In order to implement the iterative algorithm~\eqref{eq:PRSVI}, one must repeatedly solve the inner optimization problem $\max_{\xi P_\theta(\cdot |x)\in \U(x,P_\theta(\cdot |x))} \mathbb E_{\xi}[\Phi v ]$. When the state space $\mathcal X$ is large, solving this optimization problem is often computationally expensive or even intractable. Similar to Section~3.4 of~\citet{tamar2014robust}, we propose the following SAA approach to solve this problem. For the trajectory, $x_0$, $a_0$, $x_1$, $a_1$, $\ldots$, $x_{N-1}$, $a_{N-1}$, we define the empirical transition probability $P_N(x'|x,a) \doteq \frac{\sum_{t=0}^{N-1}\mathbf{1}\{x_t=x,\;a_t=a,\;x_{t+1}=x'\}}{\sum_{t=0}^{N-1}\mathbf{1}\{x_t=x,\;a_t=a\}}$\footnote{In the case when the sizes of state and action spaces are huge or when these spaces are continuous, the empirical transition probability can be found by kernel density estimation.} and $\pemp(x'|x)=\sum_{a\in\mathcal A}P_N(x'|x,a)\mu_\theta(a|x)$. Consider the following $\ell_2$-regularized empirical robust optimization problem\footnote{In the SAA approach, we only sum over the elements for which $\pemp(x'|x)>0$, thus, the sum has at most $N$ elements.}

\vspace{-0.2in}
\begin{small}
\begin{align}
\rho_N(\Phi v)&=\max_{\xi: \xi \pemp \in \U(x,\pemp)} \sum_{x'\in\mathcal X} \pemp(x'|x) \xi(x')\phi^\top(x')v \nonumber \\
&+ \frac{1}{2N}\big[\pemp(x'|x) \xi(x')\big]^2.
\label{eq:empirical-PRSVI}
\end{align}
\end{small}
\vspace{-0.2in}

As in~\citet{meng2006regularized}, the $\ell_2$-regularization term in this optimization problem guarantees convergence of optimizers $\xi^*$ and the corresponding KKT multipliers, when $N\rightarrow\infty$. Convergence of these parameters is crucial for the policy gradient analysis in the next sections. We denote by $\xi^*_{\theta,x;N}$, the solution of the above empirical optimization problem, and by $\lambda^{*,\mathcal P}_{\theta,x;N}$, $\lambda^{*,\mathcal E}_{\theta,x;N}$, $\lambda^{*,\mathcal I}_{\theta,x;N}$, the corresponding KKT multipliers.

We obtain the empirical PRSVI algorithm by replacing the inner optimization $\max_{\xi P_\theta(\cdot |x_t)\in \U(x_t,P_\theta(\cdot |x_t))} \mathbb E_{\xi}[\Phi v_\theta^* ]$ in Eq.~\ref{eq:PRSVI} with $\rho_N(\Phi v)$ from Eq.~\ref{eq:empirical-PRSVI}. Similarly, as both $N$ and $k$ tend to infinity, $v_k$ converges w.p.~1 to $v_\theta^*$. More details can be found in the supplementary material.

%%%%%%%%%%%%%%%%%%%%%%%%%%%%%%%%%%%%%%%%%%%%%%%%%%%%%%%%%%%%%%%%%%%%%%%%%%%%%%%%
%%%%%%%%%%%%%%%%%%%%%%%%%%%%%%%%%%%%%%%%%%%%%%%%%%%%%%%%%%%%%%%%%%%%%%%%%%%%%%%%
%%%%%%%%%%%%%%%%%%%%%%%%%%%%%%%%%%%%%%%%%%%%%%%%%%%%%%%%%%%%%%%%%%%%%%%%%%%%%%%%

% Actor-Critic algorithm
\subsection{Gradient Estimation}
\label{sec:pol_grad}

%\subsection{Policy gradient of $V_\theta(x)$}\label{sec:pol_grad}
In Section~\ref{sec:val_fn_rpprox}, we showed that we may effectively approximate the value function of a fixed policy $\theta$ using the (empirical) PRSVI algorithm in Eq.~\ref{eq:PRSVI}. In this section, we first derive a formula for the gradient of the Markov-coherent dynamic risk measure $\rho_\infty(\mathcal M)$, and then propose a SAA algorithm for estimating this gradient, in which we use the SAA approximation of value function from Section~\ref{sec:val_fn_rpprox}. As described in Section~\ref{subsec:Risk-Bellman}, $\rho_\infty(\mathcal M)=V_\theta(x_0)$, and thus, we shall first derive a formula for $\nabla_\theta V_\theta(x_0)$.

Let $(\xi^*_{\theta,x},\lambda^{*,\mathcal P}_{\theta,x},\lambda^{*,\mathcal E}_{\theta,x},\lambda^{*,\mathcal I}_{\theta,x})$ be the saddle point of~\eqref{eq:Lagrangian} corresponding to the state $x\in\mathcal X$. In many common coherent risk measures such as CVaR and mean semi-deviation, there are closed-form formulas for $\xi^*_{\theta,x}$ and KKT multipliers $(\lambda^{*,\mathcal P}_{\theta,x},\lambda^{*,\mathcal E}_{\theta,x},\lambda^{*,\mathcal I}_{\theta,x})$. We will briefly discuss the case when the saddle point does not have an explicit solution later in this section. Before analyzing the gradient estimation, we have the following standard assumption in analogous to Assumption \ref{ass:LR_well_behaved} of the static case.
\begin{assumption}
The likelihood ratio $\nabla_\theta\log\mu_\theta(a|x)$ is well-defined and bounded for all $x\in\mathcal X$ and $a\in\mathcal A$.
\end{assumption}

As in Theorem~\ref{thm:static_gradient} for the static case, we may use the envelope theorem and the risk-sensitive Bellman equation, $V_\theta(x)=C_\theta(x) + \gamma\max_{\xi P_\theta(\cdot |x)\in \U(x,P_\theta(\cdot |x))}\mathbb E_{\xi}[V_\theta]$, to derive a formula for $\nabla_\theta V_\theta(x)$. We report this result in Theorem~\ref{thm:dynamic_risk_supp}, which is analogous to the risk-neutral policy gradient theorem~\cite{sutton_policy_2000,konda2000actor,bhatnagar_natural_2009}. The proof is in the supplementary material.
\begin{theorem}\label{thm:dynamic_risk_supp}
Under Assumptions~\ref{assume:risk_envelope}, we have
\begin{equation*}
\nabla V_\theta(x) \!=\! \mathbb E_{\xi^*_\theta} \! \! \left[\sum_{t=0}^{\infty}\gamma^t\nabla_\theta\log\mu_\theta(\!a_t|x_t\!)h_\theta(\!x_t,\!a_t\!)\!\mid\! x_0\!\!=\!\!x\right],
\end{equation*}
where $\mathbb E_{\xi^*_\theta}[\cdot]$ denotes the expectation w.r.t.~trajectories generated by a Markov chain with transition probabilities $P_\theta(\cdot|x)\xi_{\theta,x}^*(\cdot)$, and the stage-wise cost function $h_\theta(x,a)$ is defined as

\vspace{-0.2in}
\begin{small}
\begin{align}
h_\theta(x,a) &= C(x,a) + \sum_{x'  \in \mathcal X} P(x'|x,a)\xi^*_{\theta,x}(x')\Big[\gamma V_\theta(x')-\!{\lambda}^{*,\mathcal P}_{\theta,x} \nonumber \\
&-\sum_{i\in\mathcal I} {\lambda}^{*,\mathcal I}_{\theta,x}(i)\frac{ d f_i(\xi^*_{\theta,x},p)}{d p(x')} - \sum_{e\in\mathcal E}\!{\lambda}^{*,\mathcal E}_{\theta,x}(e) \frac{ d g_e(\xi^*_{\theta,x},p)}{d p(x')}\Big].
\label{eq:h}
\end{align}
\end{small}
\vspace{-0.2in}
\end{theorem}

Theorem \ref{thm:dynamic_risk_supp} indicates that the policy gradient of the Markov-coherent dynamic risk measure $\rho_\infty(\mathcal M)$, i.e.,~$\nabla_\theta\rho_\infty(\mathcal M)=\nabla_\theta V_\theta$, is equivalent to the risk-neutral value function of policy $\theta$ in a MDP with the stage-wise cost function $\nabla_\theta\log\mu_\theta(a|x)h_\theta(x,a)$ (which is well-defined and bounded), and transition probability $P_\theta(\cdot|x)\xi_{\theta,x}^*(\cdot)$. Thus, when the saddle points are known and the state space $\mathcal X$ is not too large, we can compute $\nabla_\theta V_\theta$ using a policy evaluation algorithm. However, when the state space is large, exact calculation of $\nabla V_\theta$ by policy evaluation becomes impossible, and our goal would be to derive a sampling method to estimate $\nabla V_\theta$. Unfortunately, since the risk envelop depends on the policy parameter $\theta$, unlike the risk-neutral case, the risk sensitive (or robust) Bellman equation $T_\theta[V_\theta](x)$ in \eqref{eq:T_supp} is nonlinear in the stationary Markov policy $\mu_\theta$. Therefore $h_\theta$ cannot be considered using the action-value function ($Q$-function) of the robust MDP. Therefore, even if the exact formulation of the value function $V_\theta$ is known, it is computationally intractable to enumerate the summation over $x'$ to compute $h_\theta(x,a)$. On top of that in many applications the value function $V_\theta$ is not known in advance, which further complicates gradient estimation. To estimate the policy gradient when the value function is unknown, we approximate it by the projected risk sensitive value function $\Phi v_\theta^*$. To address the sampling issues, we propose the following \emph{two-phase sampling procedure} for estimating $\nabla V_\theta$.

{\bf (1)} Generate $N$ trajectories $\{x^{(j)}_0,a^{(j)}_0,x^{(j)}_1,a^{(j)}_1,\ldots\}_{j=1}^N$ from the Markov chain induced by policy $\theta$ and transition probabilities $P^\xi_{\theta}(\cdot|x):=\xi_{\theta,x}^*(\cdot)P_\theta(\cdot|x)$.

{\bf (2)} For each state-action pair $(x^{(j)}_t,a^{(j)}_t)=(x,a)$, generate $N$ samples $\{y^{(k)}\}_{k=1}^N$ using the transition probability $P(\cdot|x,a)$ and calculate the following empirical average estimate of $h_\theta(x,a)$

\vspace{-0.25in}
\begin{small}
\begin{align*}
h_{\theta,N}(x,&a):=C(x,a)+\frac{1}{N}\sum_{k=1}^N\xi^*_{\theta,x}(y^{(k)})\Bigg[\gamma {v_\theta^*}^\top\phi(y^{(k)})-{\lambda}^{*,\mathcal P}_{\theta,x} \\
&-\sum_{i\in\mathcal I}{\lambda}^{*,\mathcal I}_{\theta,x}(i)\frac{d f_i(\xi^*_{\theta,x},p)}{d p(y^{(k)})}-\sum_{e\in\mathcal E}{\lambda}^{*,\mathcal E}_{\theta,x}(e) \frac{d g_e(\xi^*_{\theta,x},p)}{d p(y^{(k)})}\Bigg]
\end{align*}
\end{small}
\vspace{-0.2in}

{\bf (3)} Calculate an estimate of $\nabla V_\theta$ using the following average over all the samples: $\frac{1}{N}\sum_{j=1}^N\sum_{t=0}^\infty\gamma^t\nabla_\theta\log\mu_\theta(a^{(j)}_t|x^{(j)}_t)h_{\theta,N}(x^{(j)}_t,a^{(j)}_t)$.

Indeed, by the definition of empirical transition probability $P_N(x'|x,a)$, $h_{\theta,N}(x,a)$ can be re-written as in the same structure of ${h}_{\theta}(x,a)$, except by replacing the transition probability $P(x'|x,a)$ with $P_N(x'|x,a)$.

Furthermore, in the case that the saddle points $(\xi^*_{\theta,x},\lambda^{*,\mathcal P}_{\theta,x},\lambda^{*,\mathcal E}_{\theta,x},\lambda^{*,\mathcal I}_{\theta,x}) $ do not have a closed-form solution, we may follow the SAA procedure of Section~\ref{sec:val_fn_rpprox} and replace them and the transition probabilities $P(x'|x,a)$ with their sample estimates $(\xi^*_{\theta,x;N},\lambda^{*,\mathcal P}_{\theta,x;N},\lambda^{*,\mathcal E}_{\theta,x;N},\lambda^{*,\mathcal I}_{\theta,x;N})$ and $P_N(x'|x,a)$ respectively.

At the end, we show the convergence of the above two-phase sampling procedure.
%\subsection{SAA Formulation for Policy Gradient}\label{sec:pol_approx}
Let $d_{{P}^\xi_{\theta}}(x|x_0)$ and $\pi_{{P}^\xi_{\theta}}(x,a|x_0)$ be the state and state-action occupancy measure induced by the transition probability function $P^\xi_{\theta}(\cdot|x)$, respectively. Similarly, let $d_{\pemp^\xi}(x|x_0)$ and $\pi_{\pemp^\xi}(x,a|x_0)$ be the state and state-action occupancy measure induced by the estimated transition probability function $\pemp^\xi(\cdot|x):=\xi_{\theta,x;N}^*(\cdot)P_{\theta;N}(\cdot|x)$. From the two-phase sampling procedure for policy gradient estimation and by the strong law of large numbers, when $N\rightarrow\infty$, with probability~1, we have that $\frac{1}{N}\sum_{j=1}^N\sum_{t=0}^\infty\gamma^t\mathbf 1\{x^{(j)}_t=x,a^{(j)}_t=a\}=\pi_{\pemp^\xi}(x,a|x_0)$. Based on the strongly convex property of the $\ell_2$-regularized objective function in the inner robust optimization problem $\rho_N(\Phi v)$, we can show that both the state-action occupancy measure $\pi_{\pemp^\xi}(x,a|x_0)$ and the stage-wise cost ${h}_{\theta;N}(x,a)$ converge to the their true values within a value function approximation error bound $\Delta=\|\Phi v^*_\theta-V_\theta\|_\infty$. We refer the readers to the supplementary materials for these technical results. These results together with Theorem~\ref{thm:dynamic_risk_supp} imply the consistency of the policy gradient estimation.
\begin{theorem}\label{them:consistency_dyn}
For any $x_0\in\mathcal X$, the following expression holds with probability~1:

\vspace{-0.2in}
\begin{small}
\begin{align*}
\bigg|\lim_{N\rightarrow\infty}\frac{1}{N}&\sum_{j=1}^N\sum_{t=0}^\infty\gamma^t\;\nabla\log\mu_\theta(a^{(j)}_t|x^{(j)}_t)\;h_{\theta,N}(x^{(j)}_t,a^{(j)}_t)\\
&-\nabla V_\theta(x_0)\bigg |=O(\Delta).
\end{align*}
\end{small}
\vspace{-0.2in}

\end{theorem}
\vspace{-5pt}
Thm.~\ref{them:consistency_dyn} guarantees that as the value function approximation error decreases and the number of samples increases, the sampled gradient converges to the true gradient.
%Therefore, when $\|\Phi v^*_\theta-V_\theta\|_\infty\leq\epsilon$ for arbitrarily small $\epsilon>0$, the sampled gradient  converges to true gradient $\nabla_\theta V_\theta$.

%To conclude this section, we have seen that by combining our techniques for static risk measures (Section \ref{sec:static}) with risk sensitive value-function approximation and a new policy gradient formula, we were able to extend the actor-critic method to evaluate $\dt \rho_\infty(\mdp)$. This enables us to perform effective policy search for problem \eqref{eq:DRP_problem}.
%

%%%%%%%%%%%%%%%%%%%%%%%%%%%%%%%%%%%%%%%%%%%%%%%%%%%%%%%%%%%%%%%%%%%%%%%%%%%%%%%%
%%%%%%%%%%%%%%%%%%%%%%%%%%%%%%%%%%%%%%%%%%%%%%%%%%%%%%%%%%%%%%%%%%%%%%%%%%%%%%%%
%%%%%%%%%%%%%%%%%%%%%%%%%%%%%%%%%%%%%%%%%%%%%%%%%%%%%%%%%%%%%%%%%%%%%%%%%%%%%%%%
%%%%%%%%%%%%%%%%%%%%%%%%%%%%%%%%%%%%%%%%%%%%%%%%%%%%%%%%%%%%%%%%%%%%%%%%%%%%%%%%
%%%%%%%%%%%%%%%%%%%%%%%%%%%%%%%%%%%%%%%%%%%%%%%%%%%%%%%%%%%%%%%%%%%%%%%%%%%%%%%%

\section{Convergence Analysis of Empirical PRSVI}\label{sec:SAA_VA}

\begin{lemma}[Technical Lemma]\label{lem:tech}
Let $P(\cdot|\cdot)$ and $\widetilde P(\cdot|\cdot)$ be two arbitrary transition probability matrices. At state $x\in\mathcal X$, for any $\xi\,:\,\xi P(\cdot|x)\in\mathcal{U}(x,P(\cdot|x))$, there exists a ${M}_\xi>0$ such that for some $\tilde{\xi}\,:\,\tilde{\xi}\widetilde P(\cdot|x)\in\mathcal{U}(x,\widetilde P(\cdot|x))$,
\[
\sum_{x^\prime\in \mathcal X}|\xi(x^\prime)-\tilde{\xi}(x^\prime)|\leq M_{\xi}\sum_{x^\prime\in \mathcal X}\left|P(x^\prime|x)-\widetilde P(x^\prime|x)\right|.
\]
\end{lemma}
\begin{proof}
From Theorem \ref{thm:rep}, we know that $\mathcal{U}(x,P(\cdot|x))$ is a closed, bounded,  convex set of probability distribution functions. Since any conditional probability mass function $P$ is in the interior of $\text{dom}(\mathcal{U})$ and the graph of $\mathcal{U}(x,P(\cdot|x))$ is closed, by Theorem 2.7 in \citet{rockafellar1998variational}, $\mathcal{U}(x,P(\cdot|x))$ is a Lipschitz set-valued mapping with respect to the Hausdorff distance.
Thus, for any $\xi\,:\,\xi P(\cdot|x)\in\mathcal{U}(x,P(\cdot|x))$, the following expression holds for some ${M}_\xi>0$:
\begin{equation*}
\inf_{\hat{\xi}\in\mathcal{U}(x,\widetilde P(\cdot|x))}\sum_{x^\prime\in \mathcal X}|\xi(x^\prime)-\hat{\xi}(x^\prime)|\leq {M}_\xi\sum_{x^\prime\in \mathcal X}\left| P(x^\prime|x)-\widetilde P(x^\prime|x)\right|.
\end{equation*}
Next, we want to show that the infimum of the left side is attained. Since the objective function is convex, and $\mathcal{U}(x,\widetilde P(\cdot|x))$ is a convex compact set, there exists $\tilde{\xi}\,:\,\tilde{\xi}\widetilde P(\cdot|x)\in\mathcal{U}(x,\widetilde P(\cdot|x))$ such that infimum is attained.
\end{proof}

\begin{lemma}[Strong Law of Large Number]\label{lem:SLLN_V}
Consider the sampling based PRSVI algorithm with update sequence $\{\widehat{v}_k\}$. Then as both $N$ and $k$ tend to $\infty$, $\widehat{v}_k$ converges with probability 1 to $v_\theta^*$, the
unique solution of projected risk sensitive fixed point equation $\Pi T_{\mu} [\Phi v]=\Phi v$.
\end{lemma}
\begin{proof}
By the strong law of large number of Markov process, the empirical visiting distribution and transition probability asymptotically converges to their statistical limits with probability 1, i.e.,
\[
\frac{\sum_{t=0}^{N-1}\mathbf{1}\{x_t=x\}}{N}\rightarrow d_\theta(x|x_0),\,\text{and}\,\,\widehat{P}(x'|x,a)\rightarrow P(x'|x,a),\,\forall x,x'\in\mathcal X, \,a\in\mathcal A.
\]
Therefore with probability $1$,
\[
\begin{split}
&\frac{1}{N}\sum_{t=0}^{N-1} \phi(x_t)\phi(x_t)^{\top}\rightarrow \sum_{x}d_\theta(x|x_0)\cdot\phi(x)\phi^\top(x),\\
&\frac{1}{N}\sum_{t=0}^{N-1}\phi(x_t) C_\theta(x_t)\rightarrow \sum_{x}d_\theta(x|x_0)\cdot\phi(x)C_\theta(x).
\end{split}
\]

Now we show that following expression holds with probability $1$:
\begin{equation}\label{eq:claim}
\begin{split}
&\max_{\xi\,:\,\xi\pemp(\cdot|x_t)\in\mathcal U(x_t,\pemp(\cdot|x_t))}\sum_{x'\in\mathcal X}\xi(x') \pemp(x'|x_t)v^\top\phi\left(x'\right)+\frac{1}{2N}(\xi(x')\pemp(x'|x_t))^2\\
\rightarrow&\max_{\xi\,:\,\xi P_\theta(\cdot|x_t)\in\mathcal U(x_t,P_\theta(\cdot|x_t))}\sum_{x'\in\mathcal X}\xi(x') P_\theta(x'|x_t)v^\top\phi\left(x'\right).
\end{split}
\end{equation}
Notice that for $\{\xi^*_{\theta,x_t;N}(x')\}_{x'\in\mathcal X}\in\arg\max_{\xi\,:\,\xi\pemp(\cdot|x_t)\in\mathcal U(x_t,\pemp(\cdot|x_t))}\sum_{x'\in\mathcal X}\xi(x') \pemp(x'|x_t)v^\top\phi\left(x'\right)$,  Lemma \ref{lem:tech} implies
\[
\begin{split}
&\max_{\xi\,:\,\xi\pemp(\cdot|x_t)\in\mathcal U(x_t,\pemp(\cdot|x_t))}\sum_{x'\in\mathcal X}\xi(x') \pemp(x'|x_t)v^\top\phi\left(x'\right)+\frac{1}{2N}(\xi(x')\pemp(x'|x_t))^2\\
&-\max_{\xi\,:\,\xi P_\theta(\cdot|x_t)\in\mathcal U(x_t,P_\theta(\cdot|x_t))}\sum_{x'\in\mathcal X}\xi(x') P_\theta(x'|x_t)v^\top\phi\left(x'\right)\\
\leq & \|\Phi v\|_\infty\left({M}_{\xi^*_{\theta,x_t;N}}+\max_{x\in\mathcal X}|\xi^*_{\theta,x_t;N}(x)|\right)\sum_{x^\prime\in \mathcal X}\left| P_\theta(x^\prime|x_t)-\pemp(x^\prime|x_t)\right|+\frac{1}{2N}.
\end{split}
\]
The quantity $\max_{x\in\mathcal X}|\xi^*_{\theta,x_t;N}(x)|$ is bounded because $\mathcal U(x_t,\pemp(\cdot|x_t))$ is a closed and bounded convex set from the definition of coherent risk measures. By repeating the above analysis by interchanging $P_\theta$ and $\pemp$ and combining previous arguments, one obtains
\[
\begin{split}
& \left|\max_{\xi\,:\,\xi\pemp(\cdot|x_t)\in\mathcal U(x_t,\pemp(\cdot|x_t))}\sum_{x'\in\mathcal X}\xi(x')\pemp(x'|x_t)v^\top\phi\left(x'\right)+\frac{1}{2N}(\xi(x')\pemp(x'|x_t))^2\right.\\
&\left.-\max_{\xi\,:\,\xi P_\theta(\cdot|x_t)\in\mathcal U(x_t,P_\theta(\cdot|x_t))}\sum_{x'\in\mathcal X}\xi(x')P_\theta(x'|x_t)v^\top\phi\left(x'\right)\right|\\
\leq & \|\Phi v\|_\infty\max\left\{\left({M}_{\xi^*}+\max_{x\in\mathcal X}|\xi^*(x)|\right),\left({M}_{\xi^*_{\theta,x_t;N}}+\max_{x\in\mathcal X}|\xi^*_{\theta,x_t;N}(x)|\right)\right\}\sum_{x^\prime\in \mathcal X}\left| P_\theta(x^\prime|x_t)-\pemp(x^\prime|x_t)\right|+\frac{1}{2N}.
\end{split}
\]
Therefore, the claim in expression \eqref{eq:claim} holds when $N\rightarrow\infty$ and $\sum_{x^\prime\in \mathcal X}\left| P_\theta(x^\prime|x_t)- \pemp(x^\prime|x_t)\right|\rightarrow 0$. On the other hand, the strong law of large numbers also implies that with probability $1$,
\[
\frac{1}{N}\sum_{t=0}^{N-1} \phi(x_t)\rho(\Phi v_t)\rightarrow d_\theta(x|x_0)\phi(x)\max_{\xi\,:\,\xi P_\theta(\cdot|x)\in\mathcal U(x,P_\theta(\cdot|x))}\sum_{x'\in\mathcal X}\xi(x') P_\theta(x'|x){v^*_\theta}^\top\phi\left(x'\right).
\]

Combining the above arguments implies
\[
\frac{1}{N}\sum_{t=0}^{N-1} \phi(x_t){\rho}_N(\Phi v_t)\rightarrow d_\theta(x|x_0)\phi(x)\max_{\xi\,:\,\xi P_\theta(\cdot|x)\in\mathcal U(x,P_\theta(\cdot|x))}\sum_{x'\in\mathcal X}\xi(x') P_\theta(x'|x){v_\theta^*}^\top\phi\left(x'\right) .
\]
As $N\rightarrow\infty$, the above arguments imply that $v_k-\widehat{v}_k\rightarrow 0$. On the other hand, Proposition 1  in \citet{tamar2014robust} implies that the projected risk sensitive Bellman operator $\Pi T_\theta[V]$ is a contraction, it follows that from the analysis in Section 6.3 in \citet{Ber2012DynamicProgramming} that the sequence $\{\Phi \widehat{v}_k\}$ generated by projected value iteration converges to the unique fixed point $\Phi v_\theta^*$. This in turns implies that the sequence $\{\Phi v_k\}$ converges to $\Phi v_\theta^*$.
\end{proof}

%%%%%%%%%%%%%%%%%%%%%%%%%%%%%%%%%%%%%%%%%%%%%%%%%%%%%%%%%%%%%%%%%%%%%%%%%%%%%%%%
%%%%%%%%%%%%%%%%%%%%%%%%%%%%%%%%%%%%%%%%%%%%%%%%%%%%%%%%%%%%%%%%%%%%%%%%%%%%%%%%
%%%%%%%%%%%%%%%%%%%%%%%%%%%%%%%%%%%%%%%%%%%%%%%%%%%%%%%%%%%%%%%%%%%%%%%%%%%%%%%%
%%%%%%%%%%%%%%%%%%%%%%%%%%%%%%%%%%%%%%%%%%%%%%%%%%%%%%%%%%%%%%%%%%%%%%%%%%%%%%%%
%%%%%%%%%%%%%%%%%%%%%%%%%%%%%%%%%%%%%%%%%%%%%%%%%%%%%%%%%%%%%%%%%%%%%%%%%%%%%%%%

\section{Technical Results }\label{sec:SAA_PGA}

Since by convention $\xi^*_{\theta,x;N}(x')=0$ whenever $\pemp(x'|x) =0$. In this section, we simplify the analysis by letting $\pemp(x'|x) >0$ for any $x'\in\mathcal X$ without loss of generality.
Consider the following empirical robust optimization problem:
\begin{equation}\label{eq:mid_opt}
\max_{\xi\,:\,\xi \pemp(\cdot |x)\in \U(x,\pemp(\cdot |x))} \sum_{x'\in\mathcal X} \pemp(x'|x) \xi(x')V_\theta(x'),
\end{equation}
where the solution of the above empirical problem is $\bar\xi^*_{\theta,x;N}$ and the corresponding KKT multipliers are $(\bar\lambda^{*,\mathcal P}_{\theta,x;N},\bar\lambda^{*,\mathcal E}_{\theta,x;N},\bar\lambda^{*,\mathcal I}_{\theta,x;N})$. Comparing to the optimization problem for $\rho_N(\Phi v)$, i.e.,
\begin{equation}\label{eq:rho}
\rho_N(\Phi v)=\max_{\xi\,:\,\xi \pemp(\cdot |x)\in \U(x,\pemp(\cdot |x))} \sum_{x'\in\mathcal X} \pemp(x'|x) \xi(x')\phi^\top(x')v+\frac{1}{2N}(\xi(x')\pemp(x'|x))^2,
\end{equation}
where the solution of the above empirical problem is $\xi^*_{\theta,x;N}$ and the corresponding KKT multipliers are $(\lambda^{*,\mathcal P}_{\theta,x;N},\lambda^{*,\mathcal E}_{\theta,x;N},\lambda^{*,\mathcal I}_{\theta,x;N})$,
the optimization problem in \eqref{eq:mid_opt} can be viewed as having a skewed objective function of the problem in \eqref{eq:rho}, within the deviation of magnitude $\Delta+1/2N$ where $\Delta=\|\Phi v^*_\theta-V_\theta\|_\infty$. Before getting into the main analysis, we have the following observations.
\begin{description}
\item[(i)] Without loss of generality, we can also assume $(\xi^*_{\theta,x;N},(\lambda^{*,\mathcal P}_{\theta,x;N},\lambda^{*,\mathcal E}_{\theta,x;N},\lambda^{*,\mathcal I}_{\theta,x;N}))$ follows the strict complementary slackness condition\footnote{The existence of strict complementary slackness solution follows from the KKT theorem and one can easily construct a strictly complementary pair using i.e. the Balinski-Tucker tableau with the linearized objective function and constraints, in finite time.}.
\item[(ii)] Recall from Assumption \ref{assume:risk_envelope} that the functions $f_i({\xi},p)$ and $g_e(\xi,p)$ are twice differentiable in $\xi$ at $p=P_{\theta,N}(\cdot|x)$ for any $x\in\mathcal X$.
\item[(iii)] The Slater's condition in Assumption \ref{assume:risk_envelope} implies the linear independence constraint qualification (LICQ).
\item[(iv)] Since optimization problem \eqref{eq:rho} has a convex objective function and convex/affine constraints in $\xi\in\reals^{|\mathcal X|}$, equipped with the Slater's condition we have that the first order KKT condition holds at $\xi^*_{\theta,x;N}$ with the corresponding KKT multipliers are $(\lambda^{*,\mathcal P}_{\theta,x;N},\lambda^{*,\mathcal E}_{\theta,x;N},\lambda^{*,\mathcal I}_{\theta,x;N})$. Furthermore, define the Lagrangian function
\begin{equation*}
\begin{split}
\widehat{L}_{\theta;N}(\xi,\lambda^{\mathcal P},\lambda^{\mathcal E},\lambda^{\mathcal I})\doteq&\sum_{x'\in\mathcal X} \pemp(x'|x) \xi(x')\phi^\top(x')v+\frac{1}{2N}(\pemp(x'|x) \xi(x'))^2\!\\&-\!\lambda^{\mathcal P}\left(\sum_{x' \in \mathcal X}\xi(x')P_{\theta;N}(x'|x)\!-\!1\!\right)\\
&-\sum_{e\in\mathcal E}\lambda^{\mathcal E}(e) f_e(\xi,P_{\theta;N}(\cdot|x))-\sum_{i\in\mathcal I}\lambda^{\mathcal I}(i) f_i(\xi,P_{\theta;N}(\cdot|x)).
\end{split}
\end{equation*}
One can easily conclude that $\nabla^2 \widehat{L}_{\theta;N}(\xi,\lambda^{\mathcal P},\lambda^{\mathcal E},\lambda^{\mathcal I})=-\pemp(\cdot|x)^\top\pemp(\cdot|x)/N-\sum_{i\in\mathcal I}\lambda^{\mathcal I}(i) \nabla^2_\xi f_i(\xi,P_{\theta;N}(\cdot|x))$ such that for any vector $\nu\neq 0$,
\[
\nu^\top \nabla^2\widehat{L}_{\theta;N}(\xi^*_{\theta,x;N},\lambda^{*,\mathcal P}_{\theta,x;N},\lambda^{*,\mathcal E}_{\theta,x;N},\lambda^{*,\mathcal I}_{\theta,x;N}) \nu< 0,
\]
which further implies that the second order sufficient condition (SOSC) holds at $(\xi^*_{\theta,x;N},\lambda^{*,\mathcal P}_{\theta,x;N},\lambda^{*,\mathcal E}_{\theta,x;N},\lambda^{*,\mathcal I}_{\theta,x;N})$.
\end{description}
Based on all the above analysis, we have the following sensitivity result from Corollary 3.2.4 in \cite{fiacco1983introduction}, derived based on Implicit Function Theorem.
\begin{proposition}[Basic Sensitivity Theorem]\label{prop:tech}
Under the Assumption \ref{assume:risk_envelope}, for any $x\in\mathcal X$ there exists a bounded non-singular matrix $K_{\theta,x}$ and a bounded vector $L_{\theta,x}$, such that the difference between the optimizers and KKT multipliers of optimization problem \eqref{eq:mid_opt} and \eqref{eq:rho} are bounded as follows:
\[
\begin{bmatrix}
\bar{\xi}^*_{\theta,x;N}\\
\bar{\lambda}^{*,\mathcal I}_{\theta,x;N}\\
\bar{\lambda}^{*,\mathcal P}_{\theta,x;N}\\
\bar{\lambda}^{*,\mathcal E}_{\theta,x;N}\\
\end{bmatrix}= \begin{bmatrix}
\xi^*_{\theta,x;N}\\
{\lambda}^{*,\mathcal I}_{\theta,x;N}\\
{\lambda}^{*,\mathcal P}_{\theta,x;N}\\
{\lambda}^{*,\mathcal E}_{\theta,x;N}\\
\end{bmatrix}+ \Phi_{\theta,x}^{-1}\Psi_{\theta,x}\left(\Delta+\frac{1}{2N}\right)+ o\left(\Delta+\frac{1}{2N}\right).
\]
\end{proposition}
On the other hand, we know from Proposition \ref{prop:consistent} that $\bar\xi^*_{\theta,x;N}\rightarrow \xi^*_{\theta,x}$ and $(\bar\lambda^{*,\mathcal P}_{\theta,x;N},\bar\lambda^{*,\mathcal E}_{\theta,x;N},\bar\lambda^{*,\mathcal I}_{\theta,x;N})\rightarrow(\lambda^{*,\mathcal P}_{\theta,x},\lambda^{*,\mathcal E}_{\theta,x},\lambda^{*,\mathcal I}_{\theta,x})$ with probability $1$ as $N\rightarrow\infty$.  Also recall from the law of large numbers that the sampled approximation error $\max_{x\in\mathcal X,a\in\mathcal A}\|P(\cdot|x,a)-P_N(\cdot|x,a)\|_1\rightarrow 0$ almost surely as $N\rightarrow \infty$. Then we have the following error bound in the stage-wise cost approximation $\widehat{h}_{\theta;N}(x,a)$ and $\gamma-$visiting distribution $\pi_N(x,a)$.
\begin{lemma}\label{lem:tech_2}
There exists a constant $M_h>0$ such that
$
\max_{x\in\mathcal X,a\in\mathcal A}|{h}_\theta(x,a)-\lim_{N\rightarrow\infty}\widehat{h}_{\theta;N}(x,a)|\leq M_h\Delta.
$
\end{lemma}
\begin{proof}
First we can easily see that for any state $x\in\mathcal X$ and action $a\in\mathcal A$,
\[
\begin{split}
|\widehat{h}_{\theta;N}(x,a)-{h}_\theta(x,a)|\leq &M\sum_{i\in\mathcal I}\left|\lambda^{*,\mathcal I}_{\theta,x;N}(i)-{\lambda}^{*,\mathcal I}_{\theta,x}(i)\right|+M\sum_{e\in\mathcal E}\left|\lambda^{*,\mathcal E}_{\theta,x;N}(e)-{\lambda}^{*,\mathcal E}_{\theta,x}(e)\right|+\left|\lambda^{*,\mathcal P}_{\theta,x;N}-{\lambda}^{*,\mathcal P}_{\theta,x}\right|\\
&+\gamma\|V_\theta\|_\infty\|\xi^*_{\theta,x;N}-{\xi}^*_{\theta,x}\|_1+\gamma\|V_\theta-\Phi v_\theta^*\|_\infty\\
&+\gamma\|V_\theta\|_\infty\max\{\|\xi^*_{\theta,x;N}\|_\infty,\|{\xi}^*_{\theta,x}\|_\infty\}\|P(\cdot|x,a)-P_N(\cdot|x,a)\|_1.
\end{split}
\]
Note that at $N\rightarrow \infty$, $\|P(\cdot|x,a)-P_N(\cdot|x,a)\|_1\rightarrow 0$ with probability $1$. Both $\|\xi^*_{\theta;N}\|_\infty$ and $\|{\xi}^*_{\theta,x}\|_\infty$ are finite valued because $\mathcal U(P_\theta)$ and $\mathcal U(\pemp)$ are convex compact sets of real vectors.
Therefore, by noting that $\|V_\theta\|_\infty\leq C_{\max}/(1-\gamma)$ and applying Proposition \ref{prop:consistent} and \ref{prop:tech}, the proof of this Lemma is completed by letting $N\rightarrow \infty$ and defining
\[
\begin{split}
M_h(x)=&\max\{1,M,\frac{\gamma C_{\max}}{1-\gamma}\}\left\|\begin{bmatrix}
\xi^*_{\theta,x;N}-\bar{\xi}^*_{\theta,x;N}\\
\lambda^{*,\mathcal I}_{\theta,x;N}-\bar{\lambda}^{*,\mathcal I}_{\theta,x;N}\\
\lambda^{*,\mathcal P}_{\theta,x;N}-\bar{\lambda}^{*,\mathcal P}_{\theta,x;N}\\
\lambda^{*,\mathcal E}_{\theta,x;N}-\bar{\lambda}^{*,\mathcal E}_{\theta,x;N}\\
\end{bmatrix}+\begin{bmatrix}
\bar\xi^*_{\theta,x;N}-{\xi}^*_{\theta,x}\\
\bar\lambda^{*,\mathcal I}_{\theta,x;N}-{\lambda}^{*,\mathcal I}_{\theta,x}\\
\bar\lambda^{*,\mathcal P}_{\theta,x;N}-{\lambda}^{*,\mathcal P}_{\theta,x}\\
\bar\lambda^{*,\mathcal E}_{\theta,x;N}-{\lambda}^{*,\mathcal E}_{\theta,x}\\
\end{bmatrix}\right\|_1+\gamma\Delta\\
\leq&\left(\max\{1,M,\frac{\gamma C_{\max}}{1-\gamma}\}\|\Phi_{\theta,x}^{-1}\Psi_{\theta,x}\|_1+\gamma\right)\Delta.
\end{split}
\]
\end{proof}
\begin{lemma}\label{lem:tech_3}
There exists a constant $M_\pi>0$ such that
$
\|\pi-\lim_{N\rightarrow\infty}\pi_{N}\|_1\leq M_\pi\Delta.
$
\end{lemma}
\begin{proof}
First, recall that the $\gamma-$visiting distribution satisfies the following identity:
\begin{equation}\label{eq:feas_1}
\gamma\sum_{x^\prime\in \mathcal{X}}d_{{P}^\xi_{\theta}}(x'|x)
P^\xi_{\theta}(x|x^\prime)
=d_{{P}^\xi_{\theta}}(x)-(1-\gamma)\mathbf 1\{x_0=x\},
\end{equation}
From here one easily notice this expression can be rewritten as follows:
\[
\left(I-\gamma  P^\xi_\theta\right)^\top d_{{P}^\xi_{\theta}}(\cdot|x)=\mathbf 1\{x_0=x\},\,\,\forall {x\in\mathcal X}.
\]
On the other hand, by repeating the analysis with $P_{\theta;N}(\cdot|x)$, we can also write
\[
\left(I-\gamma P^\xi_{\theta;N}\right)^\top d_{{P}^\xi_{\theta;N}}=\{\mathbf 1\{x_0=z\}\}_{z\in\mathcal X}.
\]
Combining the above expressions implies for any $x\in\mathcal X$,
\[
d_{{P}^\xi_{\theta}}-d_{{P}^\xi_{\theta;N}}-\gamma \left(\left(P^\xi_\theta\right)^\top d_{{P}^\xi_{\theta}}-(P^\xi_{\theta;N})^\top d_{{P}^\xi_{\theta;N}}\right)=0,
\]
which further implies
\[
\begin{split}
& \left(I-\gamma  P^\xi_\theta\right)^\top\left(d_{{P}^\xi_{\theta}}-d_{{P}^\xi_{\theta;N}}\right)=\gamma \left(P^\xi_\theta -P^\xi_{\theta;N} \right)^\top d_{{P}^\xi_{\theta;N}}\\
\iff & \left(d_{{P}^\xi_{\theta}}-d_{{P}^\xi_{\theta;N}}\right)=\left(I-\gamma P^\xi_\theta\right)^{-\top}\gamma \left(P^\xi_\theta -P^\xi_{\theta;N} \right)^\top d_{{P}^\xi_{\theta;N}}.
\end{split}
\]
Notice that with transition probability matrix $  P^\xi_\theta(\cdot|x)$, we have $(I-\gamma P^\xi_\theta)^{-1}=\sum_{t=0}^\infty\left(\gamma  P^\xi_\theta\right)^k<\infty$. The series is summable because by Perron-Frobenius theorem, the maximum eigenvalue of $P^\xi_\theta$ is less than or equal to $1$ and $I-\gamma P^\xi_\theta$ is invertible. On the other hand, for every given $x_0\in\mathcal X$,
\[
\begin{split}
\left\{\left(P^\xi_\theta -P^\xi_{\theta;N} \right)^\top d_{{P}^\xi_{\theta;N}}\right\}(z')=&\sum_{x\in\mathcal X}\sum_{k=0}^\infty \gamma^k(1-\gamma)\mathbb P_{{P}^\xi_{\theta;N}}(x_k=x|x_0)\left( P^\xi_\theta(z'|x)-P^\xi_{\theta;N}(z'|x)\right),\,\forall z'\in\mathcal X\\
=&\mathbb E_{{P}^\xi_{\theta;N}}\left(\sum_{k=0}^\infty \gamma^k(1-\gamma)\left( P^\xi_\theta(z'|x_k)-P^\xi_{\theta;N}(z'|x_k)\right)|x_0\right),\,\forall z'\in\mathcal X\\
\leq&\mathbb E_{{P}^\xi_{\theta;N}}\left(\sum_{k=0}^\infty \gamma^k(1-\gamma)\left| P^\xi_\theta(z'|x_k)-P^\xi_{\theta;N}(z'|x_k)\right| |x_0\right),\,\forall z'\in\mathcal X\\
\doteq& \mathcal Q(z'),\,\forall z'\in\mathcal X.
\end{split}
\]
Note that every element in matrix $(I-\gamma P^\xi_\theta)^{-1}=\sum_{t=0}^\infty\left(\gamma  P^\xi_\theta\right)^k$ is non-negative. This implies for any $z\in\mathcal X$,
\[
\begin{split}
\left|\left\{d_{{P}^\xi_{\theta}}-d_{{P}^\xi_{\theta;N}}\right\}(z)\right|=&\left|\left\{\left(I-\gamma P^\xi_\theta\right)^{-\top}\gamma \left(P^\xi_\theta -P^\xi_{\theta;N} \right)^\top d_{{P}^\xi_{\theta;N}}\right\}(z)\right|,\\
\leq &\left|\left\{\left(I-\gamma P^\xi_\theta\right)^{-\top}\gamma \mathcal Q\right\}(z)\right|=\left\{\left(I-\gamma P^\xi_\theta\right)^{-\top}\gamma \mathcal Q\right\}(z).
\end{split}
\]
The last equality is due to the fact that every element in vector $ \mathcal Q$ is non-negative. Combining the above results with Proposition \ref{prop:consistent} and \ref{prop:tech}, and noting that
\[
(I-\gamma P^\xi_\theta)^{-1}e=\sum_{t=0}^\infty\left(\gamma  P^\xi_\theta\right)^ke=\frac{1}{1-\gamma}e,
\]
we further have that
\[
\begin{split}
\|\pi-\pi_N\|_1=&\|d_{{P}^\xi_{\theta}}-d_{{P}^\xi_{\theta;N}}\|_1\\
\leq& e^\top\left(I-\gamma P^\xi_\theta\right)^{-\top}\gamma \mathcal Q\\
=&\frac{\gamma}{1-\gamma}e^\top \mathcal Q\\
\leq&\frac{\gamma}{1-\gamma}\max_{x\in\mathcal X}\left\| P^\xi_\theta(\cdot|x)-P^\xi_{\theta;N}(\cdot|x)\right\| _1\\
\leq & \frac{\gamma}{1-\gamma}\max_{x\in\mathcal X}\left(\|\xi^*_{\theta,x}(\cdot)-\xi^*_{\theta,x;N}(\cdot)\|_1\|P_\theta(\cdot|x)\|_\infty+\max\{\|\xi^*_{\theta,x;N}\|_\infty,\|{\xi}^*_{\theta,x}\|_\infty\}\|P(\cdot|x,a)-P_N(\cdot|x,a)\|_1\right),
\end{split}
\]
As in previous arguments, when $N\rightarrow \infty$, one obtains $\|P(\cdot|x,a)-P_N(\cdot|x,a)\|_1\rightarrow 0$ with probability $1$ and $\|\xi^*_{\theta,x}(\cdot)-\xi^*_{\theta,x;N}(\cdot)\|_1\rightarrow 0$. We thus set the constant $M_\pi$ as $ \gamma\|\Phi_{\theta,x}^{-1}\Psi_{\theta,x}\|_1/(1-\gamma)$.
\end{proof}

\end{document}